\documentclass[pmlr]{jmlr}

\usepackage{booktabs}
\usepackage{graphicx}
\usepackage{url}
\usepackage{amsmath}
\usepackage{algorithm2e}
\newtheorem{assumption}{Assumption}
\jmlrvolume{TBD}
\jmlryear{2026}
\jmlrworkshop{Probabilistic Graphical Models (PGM)}

\title{I-FLOP: Fast Learning of Order and Parents from Interventional Data}

\author{\Name{Liuting Chen}\textsuperscript{\textdagger} \Email{jnk758@alumni.ku.dk}\\
\AND
\Name{Alex Markham}\textsuperscript{\textdagger} \Email{alex.markham@causal.dev}\\
\addr \textsuperscript{\textdagger}Department of Mathematical Sciences\\
University of Copenhagen\\
Denmark
 }

\begin{document}

\maketitle

\begin{abstract}
	We extend the FLOP (fast learning of order and parents) algorithm recently developed by \citet{wienobst2026embracing} from observational to interventional data.
	In particular, we use the interventional BIC score of \citet{hauser2012characterization}, adapting it to be used with the iterative Cholesky-based score updates that are partly responsible for FLOP's speed.
	We show that, in the sample limit, I-FLOP recovers a DAG in the same interventional Markov equivalence class as the data-generating DAG.
	We compare I-FLOP to existing causal structure learning algorithms on simulated and real interventional data, where it performs favorably in terms of both performance and run time.
\end{abstract}


\section{Introduction}
With observational data, we can typically only determine a Markov equivalence class (MEC)~\citep{verma1990equivalence}---a set of causal graphs that cannot be distinguished from such data alone.
Interventions with known targets break some of these ambiguities, and with multiple interventional environments we obtain a finer equivalence class, the interventional Markov equivalence class (I-MEC)~\citep{hauser2012characterization}, which determines the direction of more edges than the MEC alone.

To learn causal structures from interventional data, researchers have proposed various algorithms.
These algorithms can be broadly categorized into three types based on how they utilize intervention information.
The first type is score-based methods, such as Greedy Interventional Equivalence Search (GIES)~\citep{hauser2012characterization}: by defining a scoring function that incorporates intervention information, candidate graph structures are scored, and the optimal solution is searched within the graph space.
The second type is permutation or ordering-based methods, such as Interventional Greedy Sparsest Permutation (I-GSP)~\citep{wang2017permutation}: using intervention invariance constraints to guide the inference of node order, followed by selecting parent sets under the order constraints.
The third type is invariance-based methods, such as Invariant Causal Prediction (ICP)~\citep{peters2016causal}, which leverage the principle of mechanism invariance where true causal mechanisms remain stable across multiple interventional environments while non-causal associations change. By identifying this invariance, causal relationships are pinpointed.

Among these approaches, ICP primarily tests invariance of predictive relationships for a specified response variable but is not designed as an interventional completed partially directed acyclic graph (I-CPDAG) search procedure.
We therefore focus on GIES and I-GSP, which are more directly aligned with our goal of learning a global graph or an equivalence-class representative under known intervention targets.
These two methods serve as the main reference points for our setting, but they face different challenges in practical use.

GIES uses a local greedy strategy to search the space of I-MEC, gradually adjusting the graph structure through three phases: forward, backward, and turning, which respectively consider score-improving local modifications of the current equivalence-class representative.
At each step within a phase, GIES selects the admissible local move that yields the largest improvement in the score.
While this greedy approach is computationally efficient, it does not guarantee convergence to the globally optimal structure~\citep{wang2017permutation}.
More critically, the candidate search at each step may have exponential complexity~\citep{hauser2012characterization} in the worst case, making it difficult to fully explore the search space in large-scale problems or scenarios with limited computational resources.

I-GSP uses conditional independence (CI) tests under observational and interventional distributions to guide its permutation search.
Its consistency result is asymptotic, whereas in finite samples the CI decisions can be sensitive to sample size, test choice, and significance threshold.

Motivated by these challenges, we adopt an order-based score-and-search approach for known-target interventional data.
Searching over node orderings provides an attractive alternative to direct structure search, as acyclicity is guaranteed by construction and large neighborhoods can be explored efficiently~\citep{bouckaert1992optimizing, singh1993algorithm, larranaga1996learning, friedman2003being, teyssier2005ordering, scanagatta2017improved}.
Building on this paradigm, FLOP~\citep{wienobst2026embracing} combines decomposable scoring with efficient reinsertion-based search and local parent-set updates.

We extend FLOP to known-target interventional data, yielding interventional FLOP (I-FLOP)\footnote{Code is available at \url{https://github.com/L2-T2/I-FLOP}}.
Our main contributions are threefold:
\begin{itemize}
	\item[1)] We introduce a target-filtered pooled local score that extends order-based search to multi-environment interventional data with known targets.
	\item[2)] We establish both local and global theoretical consistency guarantees: we prove that the target-filtered grow–shrink procedure consistently recovers the optimal prefix-restricted parent sets from any initialization, and that exact global minimization of our target-filtered pooled score asymptotically identifies the true I-MEC.
	\item[3)] We conduct synthetic and real-world experiments, showing that I-FLOP achieves competitive scalability and superior I-CPDAG accuracy compared with established interventional baselines, including GIES and I-GSP.
\end{itemize}

\section{Preliminaries}
Let $G=(V,E_{G})$ be a directed acyclic graph (DAG) with node set $V=\{1,\ldots,p\}$ and edge set $E_{G} \subsetneq \{i \rightarrow j \mid i,j \in V, i \neq j\}$.
We write $\operatorname{Pa}_{G}(v)$ for the parent set of node $v$.
A linear Gaussian structural equation model (SEM)~\citep{peters2014identifiability} associated with $G$ is given by \emph{structural equations} $X_v=\sum_{u\in\operatorname{Pa}_G(v)}W_{uv}X_u+\varepsilon_v$ where the noise variables $\varepsilon_v$ are jointly independent Gaussian random variables with $\varepsilon_v\sim\mathcal{N}(0,\sigma_v^2)$ and $\sigma_v^2>0$.

We consider multiple environments $\mathcal E=\{0,1,\ldots,m\}$ associated with a known intervention target family $\mathcal I=\{I_e:e\in\mathcal E\}$, where $I_0=\emptyset$ denotes the observational environment.
For a target set $I_e\subseteq V$, a known-target hard intervention replaces the structural equations of the targeted variables so that they are instead drawn jointly from an intervention distribution $Q_{I_e}^{(e)}$, i.e., $X_{I_e}^{(e)}\sim Q_{I_e}^{(e)}$, and leaves the structural equations of all non-targeted variables $v\in V \setminus I_e$ unchanged.

Every DAG can be associated to at least one topological order $\tau$ such that $u\to v\in E_{G}$ implies that $u$ precedes $v$ in $\tau$.
Given $\tau$, the candidate parent set $P_{v}$ of node $v$ consists only of nodes that precede $v$ in the order, which can be represented as $P_{v} \subseteq \operatorname{Pred}_{\tau}(v) = \{u \in V: u\; \text{precedes}\; v\; \text{in}\; \tau \}$.
The ordered Markov property states that $X_v \perp\!\!\!\perp X_{\operatorname{Pred}_\tau(v)\setminus \operatorname{Pa}_{G}(v)}\mid X_{\operatorname{Pa}_G(v)}$ for every $v \in V$.

Since the linear Gaussian SEM above has strictly positive noise variances, the induced joint Gaussian distributions are non-degenerate and hence positive.
Therefore, the ordered, local, and global Markov properties are equivalent in our setting~\citep{lauritzenGraphicalModels2026}.
A joint distribution $\mathbb{P}$ over $X_V$ is Markov with respect to $G$ if it satisfies these equivalent Markov properties.

Observational conditional independence relations do not generally identify a unique DAG: different DAGs may encode the same set of conditional independence relations and are therefore observationally indistinguishable.
Such DAGs form a Markov equivalence class (MEC), represented by a completed partially directed acyclic graph (CPDAG)~\citep{andersson1997characterization}.
Following \citet{hauser2012characterization}, two DAGs $G$ and $H$ are interventionally Markov equivalent with respect to the target family $\mathcal I$ if they induce the same interventional Markov model $M_{\mathcal I}(G)=M_{\mathcal I}(H)$.
The corresponding equivalence class is the I-MEC.
For the true DAG $G^\star$, we denote the true I-MEC by $\mathcal G^\star_{\mathcal I} = \{G\in\mathcal G_{\mathrm{DAG}}:M_{\mathcal I}(G)=M_{\mathcal I}(G^\star)\}$, where $\mathcal G_{\mathrm{DAG}}$ denotes the set of all DAGs on $V$.
The unique graphical representative of an I-MEC is the interventional essential graph, also called the I-CPDAG.

Let $\mathcal D=\{X^{(e)}\}_{e\in\mathcal{E}}$ denote the multi-environment dataset.
We propose I-FLOP as an extension of FLOP that retains FLOP's search framework for optimizing a decomposable score $S(G;\mathcal D)=\sum_{v\in V}s_v(\operatorname{Pa}_{G}(v);\mathcal D)$ over the space of topological orders (Algorithm~\ref{alg:iflop}).
Given a topological order $\tau$, each node $v$ selects its parent set $\operatorname{Pa}_{G}(v)$ only from its prefix $\operatorname{Pred}_{\tau}(v)$, thereby inducing a DAG $G$.
Candidate orders are explored by reinsertion moves (Algorithm~\ref{alg:reinsert-tf}).
During the evaluation of a reinsertion move, the affected parent sets and local scores are updated locally using a warm-start grow-shrink procedure (Algorithm~\ref{alg:growshrink-tf}), avoiding a full recomputation for all nodes.
This reinsertion-move and warm-start grow-shrink machinery is unchanged from FLOP; the component that adapts it to interventional data is the target-filtered pooled local score $s_v$, introduced next, which is the score these procedures optimize.
After the reinsertion-based local search reaches a local optimum, Iterated Local Search (ILS) applies a random perturbation to the current best order and restarts local search from the perturbed order; the best-scoring order found across these restarts is retained~\citep{lourencco2018iterated, de2003iterated}.

\section{I-FLOP: order-based search for interventional data}

\begin{algorithm2e}[H]
	\SetAlgoLined
	\DontPrintSemicolon
	\caption{Target-filtered order search for known-target interventional data}
	\label{alg:iflop}

	\KwIn{Multi-environment data $\{X^{(e)}\}_{e\in\mathcal{E}}$, target family $\mathcal{I}=\{I_e\}_{e\in\mathcal{E}}$, initial order $\tau_0$}
	\KwOut{Estimated DAG $\widehat{G}$ and its induced I-CPDAG $\widehat{C}_{\mathcal{I}}$}

	\BlankLine
	\tcp{Phase 1: Precompute environment statistics}
	\ForEach{$e \in \mathcal{E}$}{
		Center $X^{(e)}$ within environment $e$ and compute scatter matrix $C^{(e)}$\;
	}

	\BlankLine
	\tcp{Phase 2: Construct target-filtered pooled scores}
	\ForEach{$v \in V$}{
		$\mathcal{E}_v \leftarrow \{e \in \mathcal{E} : v \notin I_e\}$\;
		$N_v \leftarrow \sum_{e \in \mathcal{E}_v} n_e$\;
		$C_v^{\mathrm{pool}} \leftarrow \sum_{e \in \mathcal{E}_v} C^{(e)}$\;
		Define $s_v := s_v^{\mathrm{tf\text{-}pool}}$ from $C_v^{\mathrm{pool}}$ and $N_v$\;
	}

	\BlankLine
	\tcp{Phase 3: Initialize parent sets and local scores under $\tau_0$}
	$\tau \leftarrow \tau_0$\;
	\ForEach{$v \in V$}{
		$Z_v \leftarrow \operatorname{Pred}_{\tau}(v)$\;
		$\ell_{\emptyset,v} \leftarrow s_v(\emptyset)$\;
		$(P_v,\ell_v)
			\leftarrow
			\mathrm{GrowShrink}_{\mathrm{tf}}
			(v,Z_v,\emptyset,\ell_{\emptyset,v},0)$
		\tcp*[r]{Algorithm~\ref{alg:growshrink-tf}}
	}
	Let $P := (P_v)_{v\in V}$ and $\ell := (\ell_v)_{v\in V}$\;

	\BlankLine
	\tcp{Phase 4: Reinsertion-based local order search}
	$\mathrm{improved} \leftarrow \mathrm{true}$\;
	\While{$\mathrm{improved}$}{
		$(\tau^{\mathrm{old}},\ell^{\mathrm{old}})
			\leftarrow
			(\tau,\ell)$\;

		\ForEach{$v \in \tau^{\mathrm{old}}$}{
			$(\tau,P,\ell)
				\leftarrow
				\mathrm{Reinsert}_{\mathrm{tf}}(\tau,P,\ell,v)$
			\tcp*[r]{Algorithm~\ref{alg:reinsert-tf}}
		}

		$\mathrm{improved}
			\leftarrow
			\left(
			\sum_{v\in V}\ell_v
			<
			\sum_{v\in V}\ell_v^{\mathrm{old}}
			\right)$\;
	}

	\BlankLine
	\tcp{Phase 5: Graph construction and output}
	$\widehat{G}
		\leftarrow
		\bigl(
		V,
		\{u\to v : u\in P_v,\ v\in V\}
		\bigr)$\;

	$\widehat{C}_{\mathcal{I}}
		\leftarrow
		\operatorname{I\text{-}CPDAG}
		(\widehat{G},\mathcal{I})$\;

	\BlankLine
	\KwRet{$(\widehat{G},\widehat{C}_{\mathcal{I}})$}\;
\end{algorithm2e}

Algorithm~\ref{alg:iflop} presents the core reinsertion-based local order search. In our implementation, the optional ILS wrapper described above repeatedly perturbs the locally optimal order and reruns the reinsertion search, retaining the best-scoring solution across restarts.

We first state two assumptions for the known-target hard-intervention setting: one on the data-generating process and one on the conservativeness of the intervention target family.

\begin{assumption}[Data-generating process and causal sufficiency]
	\label{assump:data-generating process}

	We assume causal sufficiency over $V$ and a linear Gaussian SEM with true DAG $G^\star=(V,E)$.
	In each environment $e\in\mathcal{E}$, the known hard-intervention targets $I_e\subseteq V$ satisfy $X^{(e)}_{I_e}\sim Q^{(e)}_{I_e}$, while each non-targeted node $v\notin I_e$ retains its structural equation $X^{(e)}_v=\sum_{u\in\operatorname{pa}_{G^\star}(v)}W^\star_{uv}X^{(e)}_u+\varepsilon^{(e)}_v$, where $\varepsilon^{(e)}_v\sim\mathcal{N}(0,(\sigma^\star_v)^2)$.
	The noise variables are mutually independent, have positive variances, and are invariant across environments.
\end{assumption}

\begin{assumption}[Conservative family of targets]
	\label{assump:conservative intervention design}

	Following \citet{hauser2012characterization}, we assume the intervention target family $\mathcal{I}=\{I_e:e\in\mathcal{E}\}$ is conservative, meaning that for every node $v\in V$, there exists at least one environment $e\in\mathcal{E}$ such that $v \notin I_e$.
	Intuitively, this assumption ensures that no variable is perpetually perturbed across all experimental settings, so that the unperturbed, natural causal mechanism $P(v \mid \operatorname{pa}(v))$ of each variable can be observed and learned from at least one environment.
\end{assumption}

\subsection{Valid environments under interventions}

In the observational setting, FLOP scores each candidate parent set by fitting a single local regression on samples assumed to come from one observational regime.
This is no longer appropriate for interventional multi-environment data: different environments may induce different joint distributions, and samples from environments where node $v$ is directly intervened upon do not follow the original mechanism of $v$.
The invariant object is instead the local mechanism of $v$ across environments in which $v$ is not targeted.
This leads to response-target filtering:

\begin{definition}
	For each node $v\in V$, we define the valid environment set $\mathcal E_v := \{e\in\mathcal E: v\notin I_e\}$, which consists of the environments in which $v$ is not directly intervened upon, and the corresponding total sample size is $N_{v}:= \sum_{e\in\mathcal E_v} n_e$.
\end{definition}

We use $N$ to denote the total sample size over all environments.
The valid environment set $\mathcal E_v$ is node-specific: different response nodes may use different subsets of environments for local scoring.
Importantly, the validity of an environment for node $v$ depends only on whether $v$ itself is directly intervened upon.
In particular, let $P\subseteq V\setminus\{v\}$ be a candidate parent set for $v$.
If $u\in P$ is directly intervened upon in environment $e$, while $v\notin I_e$, then $e\in\mathcal E_v$ remains valid for scoring $v$.
By Assumption~\ref{assump:data-generating process}, intervening on $u$ may change the distribution of $X_u^{(e)}$, but it does not replace the structural equation of $v$ when $v\in V \setminus I_e$.
Therefore, observations from all environments in $\mathcal E_v$ can be pooled to estimate the same structural equation for $v$.
We next use these node-specific valid environments to define the target-filtered pooled local score.

\subsection{Target-filtered pooled local score}
\begin{definition}
	For node \(v\in V\) and candidate parent set \(P\subseteq V\setminus\{v\}\), define the target-filtered pooled local score by
	$$
		s_{v}^{\mathrm{tf\text{-}pool}}(P) = \frac{N_{v}}{2} \left( 1+\log\widehat\sigma^{2}_{v\mid P} \right) + \frac{1}{2}\log N\,(|P|+1).
	$$
	Here, $\widehat{\sigma}^{2}_{v\mid P}$ is estimated only from the valid environments $\mathcal{E}_{v}$.
\end{definition}

For each environment, the centered data and scatter matrix are
$$
	\widetilde X^{(e)} = X^{(e)} - \bar X^{(e)},\qquad
	C^{(e)} = (\widetilde X^{(e)})^\top \widetilde X^{(e)},
$$
and the pooled scatter matrix for node $v$ over its valid environments is $C_v^{\mathrm{pool}} = \sum_{e \in \mathcal E_v} C^{(e)}$. We impose the following non-singularity assumption on the relevant principal submatrices of $C_v^{\mathrm{pool}}$.

\begin{assumption}[Local non-singularity]
	\label{assump:local-nonsingularity}
	For every node $v\in V$ and every candidate parent set $P\subseteq V\setminus\{v\}$ considered in the local score evaluation, the principal submatrix $(C_v^{\mathrm{pool}})_{PP}$ is non-singular.
\end{assumption}
Under Assumption~\ref{assump:local-nonsingularity}, the pooled residual variance estimator $\widehat\sigma^{2}_{v\mid P}$ for node $v$ and candidate parent set $P$ is defined as
$$
	\widehat \sigma^2_{v\mid P} =
	\frac{1}{N_v}\Big[ (C_v^{\mathrm{pool}})_{vv} - (C_v^{\mathrm{pool}})_{vP} ((C_v^{\mathrm{pool}})_{PP})^{-1} (C_v^{\mathrm{pool}})_{Pv} \Big],
$$
where the second term is set to zero when $P=\emptyset$ (the Schur complement of the pooled scatter matrix; see the Frisch--Waugh--Lovell identity used in Appendix~\ref{app:proofs}).
Thus, $\widehat\sigma^2_{v\mid P}$ corresponds to the residual variance obtained by fitting a shared local linear mechanism for $v$ across its valid environments, rather than across all environments.

By construction, this score preserves node-wise decomposability and is compatible with order-restricted parent sets.
For a candidate DAG $G'$, the total score is
$$
	S_{\mathrm{tf\text{-}pool}}(G'; \mathcal D, \mathcal I) = \sum_{v \in V} s_v^{\mathrm{tf\text{-}pool}}(\mathrm{Pa}_{G'}(v)).
$$

Given a fixed order \(\tau\), score-based I-FLOP selects parent sets by
$$
	P_v(\tau) \in \operatorname*{arg\,min}_{P \subseteq \mathrm{Pred}_\tau(v)} s_v^{\mathrm{tf\text{-}pool}}(P),
$$
and induces the order-restricted DAG $G_\tau = \{ u \to v : u \in P_v(\tau) \}$.

The corresponding order objective is
$$
	Q_{\mathrm{score}}(\tau) = \sum_{v \in V} s_v^{\mathrm{tf\text{-}pool}}(P_v(\tau)).
$$

\subsection{Warm-start parent-set selection}

For each node $v$ and prefix $Z$, I-FLOP updates the candidate parent set using a grow–shrink procedure with the target-filtered pooled local score (Algorithm~\ref{alg:growshrink-tf}).\footnote{Proofs of all theoretical results in this section are provided in Appendix~\ref{app:proofs}.}
This strategy is inherited from FLOP, which recovers the restricted Markov boundary under a single observational distribution.
In our setting, however, valid environments preserve the response mechanism but may exhibit different joint distributions due to interventions on other nodes.
Consequently, the restricted Markov boundary characterization does not directly apply to multiple data sources.
To characterize the asymptotic objective optimized by the grow--shrink procedure, we introduce a pooled population residual risk over the valid environments.
The following assumption specifies the asymptotic sample proportions contributed by each environment.

\begin{assumption}[Environment proportions]
	\label{assump:env-proportions}
	For every environment \(e\in\mathcal E\), the sample proportion satisfies $n_e/N\to\pi_e$, where $\pi_e>0$.
	Consequently, for every node $v$, $N_v/N\to\alpha_v:=\sum_{e\in\mathcal E_v}\pi_e>0$.
\end{assumption}

For each environment $e\in\mathcal E$, let $\Sigma^{(e)} = \operatorname{Cov}\!\left(\widetilde X^{(e)}\right)$ denote the covariance matrix of the environment-centered variables.
Under Assumptions~\ref{assump:data-generating process}, \ref{assump:conservative intervention design}, and \ref{assump:env-proportions}, the normalized valid-environment pooled scatter matrix $N^{-1}_{v}C_v^{\mathrm{pool}}$ converges in probability to $\Sigma_v^{\mathrm{pool}} = \frac{1}{\alpha_v} \sum_{e\in\mathcal E_v} \pi_e\Sigma^{(e)}$ (by the weak law of large numbers applied within each valid environment, combined via Assumption~\ref{assump:env-proportions}'s sample-proportion limits).
For a candidate parent set $P\subseteq Z$, the corresponding population pooled residual risk is
$$
	\sigma^2_{v\mid P,\mathrm{pool}}
	=
	(\Sigma_v^{\mathrm{pool}})_{vv}
	-
	(\Sigma_v^{\mathrm{pool}})_{vP}
	\bigl(
	(\Sigma_v^{\mathrm{pool}})_{PP}
	\bigr)^{-1}
	(\Sigma_v^{\mathrm{pool}})_{Pv}.
$$
This quantity is the population counterpart of $\widehat{\sigma}^2_{v\mid P}$ and characterizes the asymptotic fit of the shared response mechanism across the valid environments.
It therefore induces the population target of the grow--shrink procedure:

\begin{definition}
	The pooled local parent target of $v$ relative to $Z$ is defined as the sparsest candidate parent set in $Z$ that gives the optimal valid-environment structural-mechanism fit:
	\[
		M_v^{\mathrm{pool}}(Z)
		\in
		\operatorname*{arg\,min}_{P\subseteq Z}
		\sigma^2_{v\mid P,\mathrm{pool}},
	\]
	where ties are resolved in favor of a sparsest minimizer.
\end{definition}
Although single-variable greedy search procedures could in principle be susceptible to local optima, the underlying linear Gaussian model and positive definiteness of the pooled covariance guarantee that the pooled residual-risk landscape has no local traps:

\begin{proposition}
	\label{prop:no-local-traps}
	Fix $v\in V$ and $Z\subseteq V\setminus\{v\}$.
	Under Assumptions~\ref{assump:data-generating process}, \ref{assump:conservative intervention design}, and \ref{assump:env-proportions}, the pooled residual risk criterion $R(P) = \sigma^2_{v\mid P,\mathrm{pool}}$ has a unique sparsest minimizer $M_v^{\mathrm{pool}}(Z) = \operatorname{supp}(w_Z)$, where $w_Z = (\Sigma_{ZZ}^{\mathrm{pool}})^{-1}\Sigma_{Zv}^{\mathrm{pool}}$.
	Furthermore, for every candidate set $P\subseteq Z$ with $M_v^{\mathrm{pool}}(Z)\nsubseteq P$, there exists $u\in M_v^{\mathrm{pool}}(Z)\setminus P$ satisfying $\sigma^2_{v\mid P\cup\{u\},\mathrm{pool}} < \sigma^2_{v\mid P,\mathrm{pool}}$.
\end{proposition}

Proposition~\ref{prop:no-local-traps} establishes that whenever the current candidate parent set $P$ does not fully cover the optimal parent set $M_v^{\mathrm{pool}}(Z)$, there always exists at least one missing parent whose addition strictly reduces the pooled residual risk.
Consequently, target-filtered grow--shrink provably converges to the global parent target from any initialization:

\begin{theorem}
	\label{thm:iflop-gs-consistency}
	Fix $v\in V$ and $Z\subseteq V\setminus\{v\}$.
	Let $\widehat{P}_{v,N}(Z)$ be the parent set returned by target-filtered grow--shrink using $s_v^{\mathrm{tf\text{-}pool}}$.
	Under Assumptions~\ref{assump:data-generating process}--\ref{assump:env-proportions}, starting from any initial parent set $P_0\subseteq Z$,
	$$
		\mathbb{P}\left(\widehat{P}_{v,N}(Z) = M_v^{\mathrm{pool}}(Z)\right) \longrightarrow 1 \quad \text{as } N \to \infty.
	$$
\end{theorem}

Since a fixed order $\tau$ induces a fixed-prefix problem for each node, applying Theorem~\ref{thm:iflop-gs-consistency} node-wise immediately yields the following fixed-order recovery guarantee:

\begin{corollary}
	\label{cor:iflop-fixed-order}
	Fix an order $\tau$, and let $Z_v(\tau)=\operatorname{Pred}_\tau(v)$ for each $v\in V$. For each node $v$, let $\widehat{P}_{v,N}(\tau)$ be the parent set returned by target-filtered grow--shrink with prefix $Z_v(\tau)$, and define the induced graph $\widehat{G}_N(\tau) = \{u\to v: u\in\widehat{P}_{v,N}(\tau),\ v\in V\}$. Then, under Assumptions~\ref{assump:data-generating process}--\ref{assump:env-proportions},
	\[
		\mathbb{P}\left(\mathrm{Pa}_{\widehat{G}_N(\tau)}(v) = M_v^{\mathrm{pool}}(Z_v(\tau)) \text{ for all } v\in V\right) \longrightarrow 1 \quad \text{as } N \to \infty.
	\]
\end{corollary}

Theorem~\ref{thm:iflop-gs-consistency} and Corollary~\ref{cor:iflop-fixed-order} establish the large-sample behavior of the target-filtered grow--shrink updates for a fixed prefix and a fixed order, respectively.
These results justify the local parent-set selection used in I-FLOP.
We next show that the target-filtered pooled score admits a valid likelihood interpretation under known interventions and use this connection to state its global recovery property.

\subsection{Global recovery from the target-filtered pooled score}
Under known hard interventions, the replacement mechanism of a directly targeted node does not depend on its candidate parent set in the original structural mechanism and therefore contributes only a graph-independent nuisance term.
After profiling out environment-specific intercepts, equivalently by working with the environment-centered observations, let
\[
	\ell_N(G,\theta_G) = \sum_{v\in V} \sum_{e\in\mathcal E_v} \ell_{v,e} \bigl( \operatorname{Pa}_G(v);\theta_v \bigr),
\]
denote the resulting log-likelihood contribution of the original structural mechanisms, evaluated only over the valid environments of each response node.
Here, $\theta_G=\{\theta_v:v\in V\}$, where $\theta_v$ contains the regression coefficients associated with $\operatorname{Pa}_G(v)$ and the residual variance of the original mechanism of node $v$.

\begin{lemma}
	\label{lem:likelihood-compatibility}
	Under Assumptions~\ref{assump:data-generating process}--\ref{assump:local-nonsingularity}, the target-filtered pooled graph score is the decomposable score associated with $\ell_N(G,\theta_G)$, up to terms independent of the candidate DAG.
	In particular, for any candidate DAG $G'$,
	\[
		S_{\mathrm{tf\text{-}pool}}(G';\mathcal D,\mathcal I) = - \sup_{\theta_{G'}}\ell_N(G',\theta_{G'}) + \frac{1}{2}\log N \sum_{v\in V} \bigl( |\operatorname{Pa}_{G'}(v)| + 1 \bigr) + C_N.
	\]
	where $C_N = - \frac{1}{2} \sum_{v\in V} N_v\log(2\pi)$ does not depend on $G'$.
\end{lemma}

We now characterize the global target of exact score minimization in terms of I-MEC.
For a candidate DAG $G$, define the population counterpart of the graph-dependent fit component of the target-filtered pooled score as
\[
	L_\infty(G) = \sum_{v\in V} \frac{\alpha_v}{2} \log \sigma^2_{v\mid\operatorname{Pa}_G(v),\mathrm{pool}}, \qquad G\in\mathcal G_{\mathrm{DAG}}.
\]
Let
\[
	\mathcal A = \operatorname*{arg\,min}_{G\in\mathcal G_{\mathrm{DAG}}} L_\infty(G), \qquad \mathcal A_{\min} = \operatorname*{arg\,min}_{G\in\mathcal A} |E_G|.
\]
The restriction to $\mathcal A_{\min}$ corresponds to the graph-dependent complexity term
\[
	\sum_{v\in V} \bigl( |\operatorname{Pa}_G(v)|+1 \bigr) = |E_G|+|V|.
\]

By definition, $\mathcal A$ contains the DAGs attaining the optimal population fit of the valid-environment structural mechanisms, and $\mathcal A_{\min}$ retains the sparsest DAGs among these representations.
To conclude that the members of $\mathcal A_{\min}$ belong to the true $\mathcal I$-MEC, we require the conditional independence relations induced by the known intervention targets to coincide with those encoded by the true DAG $G^\star$.
We therefore impose the following interventional faithfulness assumption.

\begin{assumption}[Interventional faithfulness assumption]
	\label{assump:interventional faithfulness assumption}

	\noindent Let $\mathcal G_{\mathcal I}^{\star} := \left\{ G\in \mathcal G_{\mathrm{DAG}}: M_{\mathcal I}(G) = M_{\mathcal I}(G^\star) \right\}$ denote the true I-MEC.
	The interventional distribution family generated by $G^\star$ is faithful to $G^\star$ with respect to the known intervention target family $\mathcal I$.
\end{assumption}

\begin{theorem}
	\label{thm:global-imec-identifiability}
	Under Assumptions~\ref{assump:data-generating process}, \ref{assump:conservative intervention design}, \ref{assump:local-nonsingularity}, \ref{assump:env-proportions} and \ref{assump:interventional faithfulness assumption}, the sparsest global residual-risk minimizers belong to the true I-MEC:
	\[
		\mathcal A_{\min}
		\subseteq
		\mathcal G_{\mathcal I}^{\star}
		=
		\left\{
		G\in\mathcal G_{\mathrm{DAG}}:
		M_{\mathcal I}(G)=M_{\mathcal I}(G^\star)
		\right\}.
	\]
\end{theorem}

Theorem~\ref{thm:global-imec-identifiability} is a global statement.
For the empirical score, uniform convergence of the target-filtered residual risks and the BIC sparsity penalty imply that an exact global minimizer selects a DAG in $\mathcal A_{\min}$ in the limit $N\to\infty$.
Combining this with $\mathcal A_{\min}\subseteq\mathcal G_{\mathcal I}^\star$ gives the following I-CPDAG recovery result:

\begin{corollary}
	\label{cor:exact-icpdag-recovery}
	Let $\widehat G_N$ be an exact global minimizer of $S_{\mathrm{tf\text{-}pool}}(G;\mathcal D,\mathcal I)$ over DAG $G$.
	Under Assumptions~\ref{assump:data-generating process}, \ref{assump:conservative intervention design}, \ref{assump:local-nonsingularity}, \ref{assump:env-proportions}, and \ref{assump:interventional faithfulness assumption}, $\widehat G_N$ belongs to the true I-MEC with probability tending to one.
	If the target family $\mathcal I$ fully identifies the DAG, so that $\mathcal G_{\mathcal I}^\star=\{G^\star\}$, then $\Pr(\widehat G_N=G^\star)\to1$.
\end{corollary}

\section{Experiments}
In this section, we evaluate I-FLOP for interventional causal discovery.
We compare it with two external baselines, GIES and I-GSP, and two FLOP-based variants: FLOP-obs, the original observational method, and FLOP-envwise, which uses environment-wise variance estimation without target-based environment filtering.
We generate synthetic data by Erd\H{o}s--R\'enyi (ER) DAGs~\citep{erdos1959pmd} under linear Gaussian SEMs.
Each dataset contains one observational environment and multiple known-target hard interventional environments.

The experiment can be divided into three main parts.
First, we assess finite-sample convergence, specifically whether the algorithms recover the I-CPDAG (Section~\ref{finite sample}).
Second, we evaluate the algorithms' performances under well-specified synthetic settings (Section~\ref{main comparison}).
Finally, we demonstrate performance on real interventional data from the Causal Chamber Light Tunnel system (Section~\ref{causalchamber}).

\subsection{Finite sample convergence toward the asymptotic guarantee}
\label{finite sample}
We first investigate the finite-sample convergence behavior of I-FLOP as the sample size increases.
Our theoretical results establish that, in the large-sample limit, optimizing the target-filtered pooled score recovers the true I-CPDAG.
This experiment therefore examines whether the structural estimates produced by I-FLOP become increasingly accurate and stable with more data.
Figure~\ref{fig:consistency} shows that I-FLOP recovers the I-CPDAG in large sample limit, which supports our claim.

\begin{figure}[t]
	\centering
	\includegraphics[width=0.7\linewidth]{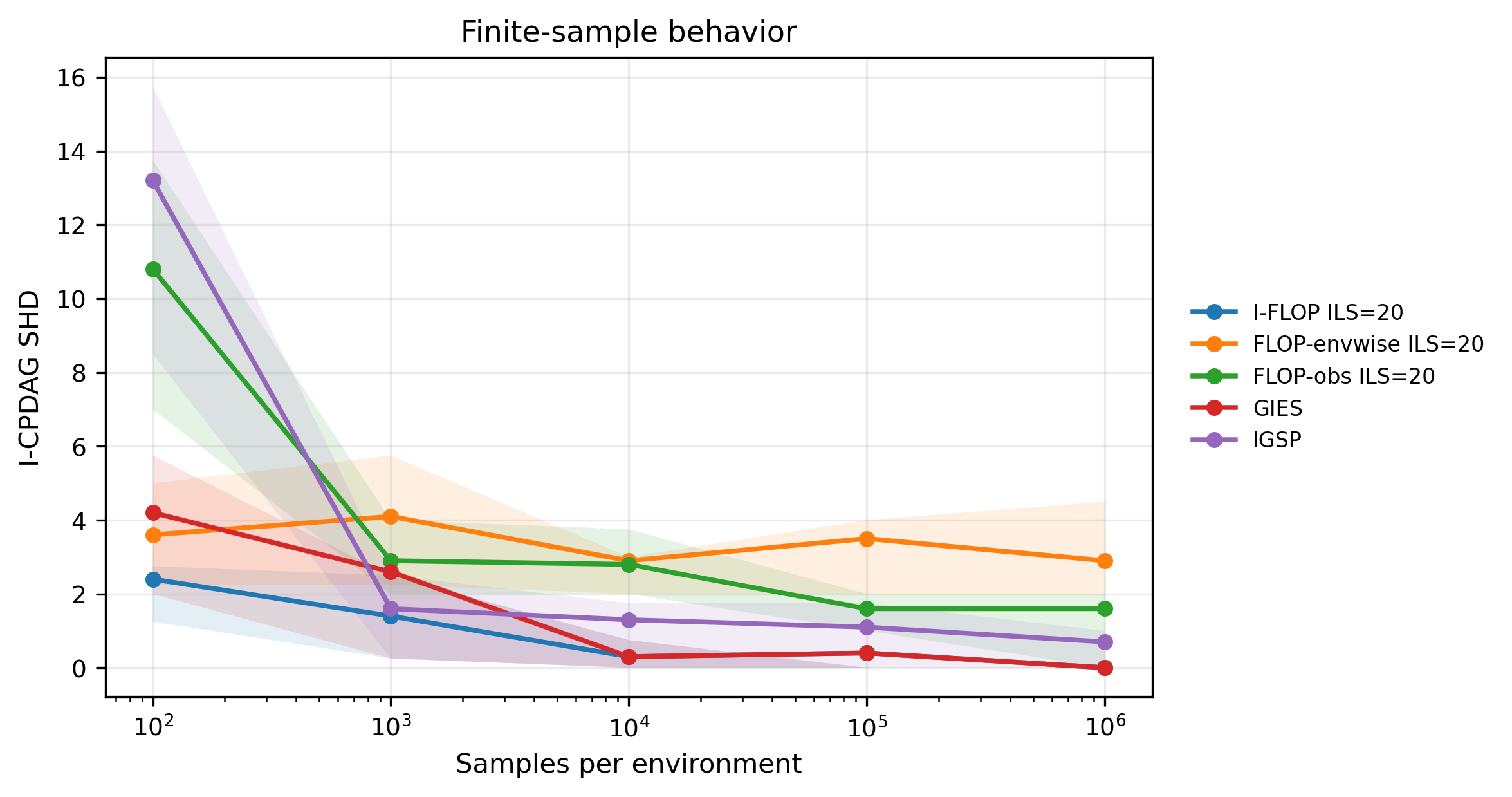}
	\caption{I-CPDAG SHD decreases as sample size increases from $10^2$ to $10^6$ per environment with 20 nodes and expected degree 2 for one observational and five known-target singleton hard intervention environments.}
	\label{fig:consistency}
\end{figure}

\subsection{Main accuracy and runtime comparison}
\label{main comparison}
\begin{figure}[t]
	\centering
	\includegraphics[width=\textwidth]{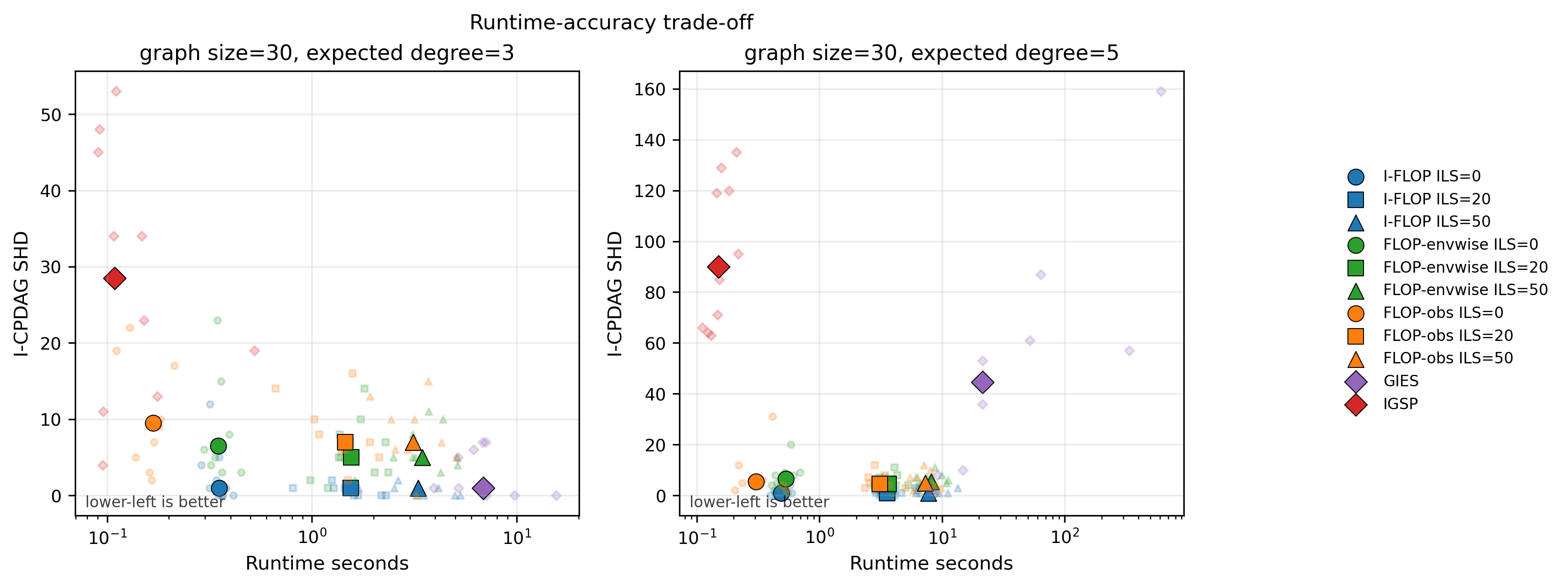}
	\caption{The I-CPDAG Structural Hamming Distance (SHD) is plotted against runtime for linear-Gaussian ER DAGs.
		Each setting uses one observational environment and five known-target singleton hard intervention environments, with $1{,}000$ samples per environment.}
	\label{fig:simulation}
\end{figure}

\begin{table}[ht]
	\centering
	\caption{SHD decomposition of experiment in well-specified synthetic setting}
	\label{tab:simulation}
	\setlength{\tabcolsep}{4pt}
	\small
	\begin{tabular}{lrrrrrrrr}
		\toprule
		                & \multicolumn{4}{c}{Expected degree $d=3$}
		                & \multicolumn{4}{c}{Expected degree $d=5$}                           \\
		\cmidrule(lr){2-5} \cmidrule(lr){6-9}
		Algorithm
		                & Mean                                      & Std.  & Skel. & Orient.
		                & Mean                                      & Std.  & Skel. & Orient. \\
		\midrule
		I-FLOP (ILS=0)  & 2.50                                      & 3.78  & 1.30  & 1.20
		                & 1.60                                      & 1.58  & 0.90  & 0.70    \\
		I-FLOP (ILS=20) & 0.70                                      & 0.67  & 0.60  & 0.10
		                & 1.50                                      & 1.27  & 1.10  & 0.40    \\
		I-FLOP (ILS=50) & 0.70                                      & 0.67  & 0.60  & 0.10
		                & 1.50                                      & 1.27  & 1.10  & 0.40    \\
		GIES            & 2.70                                      & 3.13  & 2.00  & 0.70
		                & 47.40                                     & 49.02 & 39.10 & 8.30    \\
		I-GSP           & 28.40                                     & 16.93 & 19.60 & 8.80
		                & 94.70                                     & 28.79 & 73.10 & 21.60   \\
		\bottomrule
	\end{tabular}
\end{table}

Figure~\ref{fig:simulation} compares accuracy and runtime in a representative well-specified synthetic setting.
I-FLOP achieves the lowest I-CPDAG SHD in both ER graph regimes, with a clearer advantage on denser graphs, where GIES and I-GSP have substantially larger structural errors.
I-FLOP is also markedly faster than GIES in the denser setting while remaining substantially more accurate than the faster I-GSP baseline, yielding a favorable accuracy--runtime trade-off.
Table~\ref{tab:simulation} shows the decomposition comparison of I-CPDAG SHD between our I-FLOP and the external baselines.

\subsection{Real-data evaluation on Causal Chamber}
\label{causalchamber}

We evaluate algorithms on real interventional data from the Causal Chamber Light Tunnel system~\citep{gamella2025causal}. The view we use contains 20 variables, including one observational environment (10,000 samples) and five singleton hard-intervention environments with known targets (red, green, blue, \texttt{pol\_1}, \texttt{pol\_2}, each containing 1,000 samples), all derived from \texttt{lt\_interventions\_standard\_v1}.
Figure~\ref{fig:chamber} shows that I-FLOP exhibits the best accuracy--runtime trade-off on this real interventional dataset.

\begin{figure}[ht]
	\centering
	\includegraphics[width=0.7\linewidth]{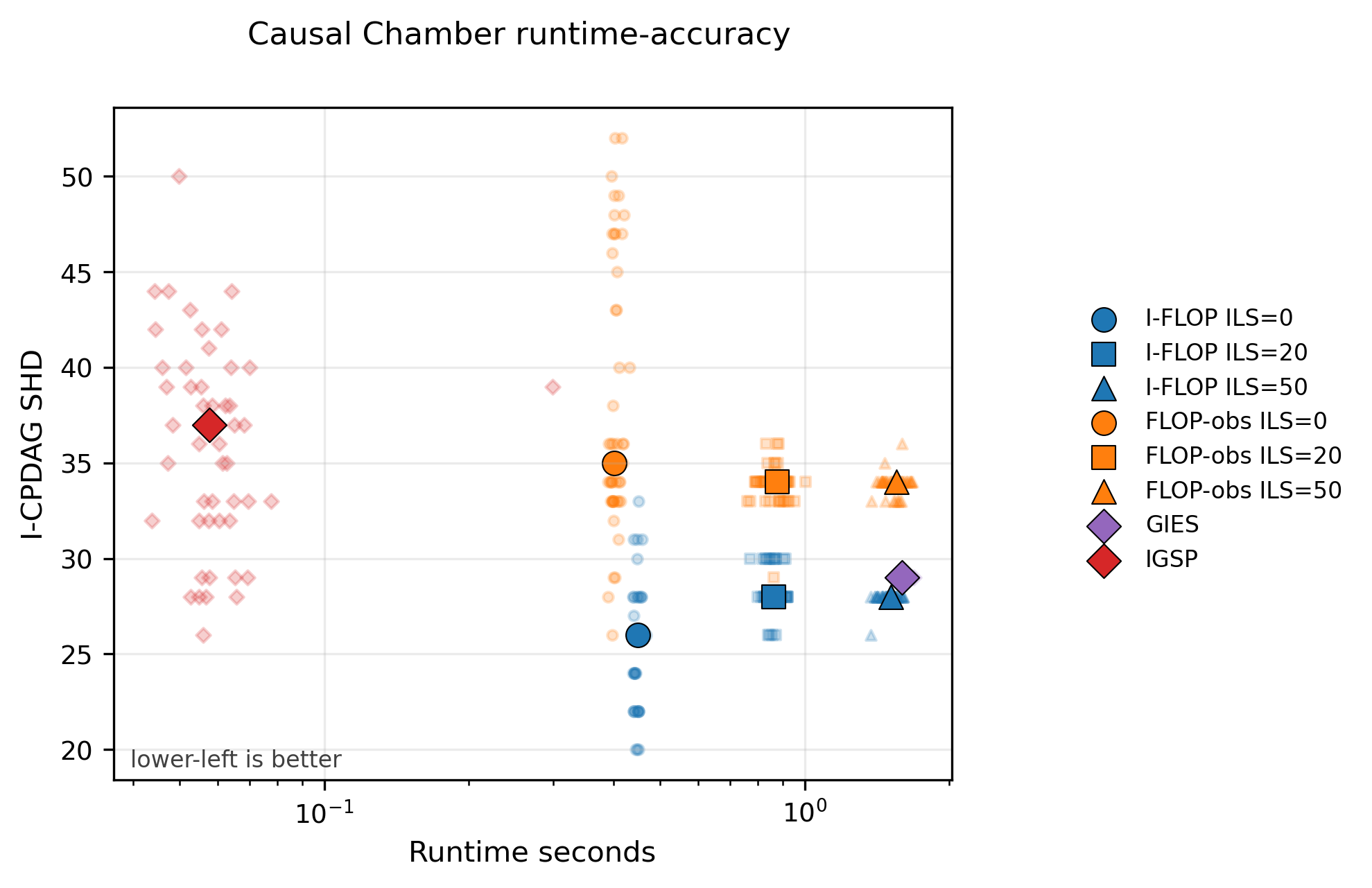}
	\caption{I-FLOP (ILS = 0) attains the lowest mean I-CPDAG SHD while remaining faster than GIES.
		All I-FLOP configurations outperform FLOP-obs and I-GSP, and positive ILS budgets reduce variability (details in Appendix~\ref{ILS}).}
	\label{fig:chamber}
\end{figure}

\section{Conclusion}
\label{sec:conclusion}

We introduce I-FLOP, an order-based score search method for causal structure learning from multi-environment interventional data with known targets. By replacing FLOP's observational local score with a target-filtered pooled interventional score, I-FLOP retains the efficiency of warm-start grow--shrink updates and reinsertion-based order search while incorporating intervention information in a principled manner. Under Assumptions~\ref{assump:data-generating process}--\ref{assump:interventional faithfulness assumption}, exact global minimization of the target-filtered pooled score asymptotically recovers the true I-MEC, while under Assumptions~\ref{assump:data-generating process}--\ref{assump:env-proportions}, the fixed-prefix target-filtered grow--shrink update recovers the pooled local parent target in the large-sample limit. Empirically, I-FLOP achieves a favorable accuracy--runtime trade-off on synthetic interventional data and competitive performance on the Causal Chamber benchmark.

The main limitations of the current framework stem from its reliance on known hard interventions, linear-Gaussian mechanisms with strictly positive error variances, and conservative intervention designs. These assumptions restrict the model class to stochastic linear systems and exclude deterministic or functional relationships with degenerate noise, as well as nonlinear or uncertain intervention mechanisms. This mismatch is particularly relevant for physical systems such as the Causal Chamber, where variables may obey deterministic conservation laws or bounded nonlinear responses, and helps explain why more aggressive score optimization does not necessarily yield lower SHD when the data-generating process departs from the assumed model class. Future work could therefore extend target-filtered order search to deterministic or functional relationships, nonlinear additive noise models, and soft, uncertain, or partially known intervention targets (i.e., settings where the identity of the intervened variable(s) in a given environment is not fully or reliably known, as considered by \citet{squires2020permutation} and \citet{gamella2022characterization}), potentially through more flexible non-parametric or generalized scoring criteria.

\section*{Acknowledgments}
A.M.~was partially supported by Novo Nordisk Foundation grant NNF20OC0062897 and also by the Pioneer Centre for Statistical and computational Methods for Advanced Research to Transform Biomedicine (SMARTbiomed), DNRF grant number P4.
We thank Ming Cai from the Graduate School of Informatics, Kyoto University, for helpful discussions and comments on an earlier version of the manuscript.

\bibliography{references}

\newpage
\appendix
\section{Proofs of Theoretical Results}
\label{app:proofs}

\subsection{Single-Move Score Comparisons}

We first establish the large-sample score comparisons which will be used in the proof of Theorem~\ref{thm:iflop-gs-consistency}.

\begin{lemma}
	\label{lem:tf-local-score-comparison}
	Fix $v\in V$ and $Z\subseteq V\setminus\{v\}$.
	Under Assumptions~\ref{assump:data-generating process}--\ref{assump:env-proportions}, the following statements hold simultaneously over the finitely many candidate parent sets contained in $Z$ with probability tending to one.
	\begin{enumerate}
		\item If $P,Q\subseteq Z$ satisfy $\sigma^{2}_{v\mid Q,\mathrm{pool}} < \sigma^{2}_{v\mid P,\mathrm{pool}}$, then $s_{v}^{\mathrm{tf\text{-}pool}}(Q) < s_{v}^{\mathrm{tf\text{-}pool}}(P)$.

		\item If $u\in P\subseteq Z$ and $\sigma^{2}_{v\mid P\setminus\{u\},\mathrm{pool}} = \sigma^{2}_{v\mid P,\mathrm{pool}}$, then $s_{v}^{\mathrm{tf\text{-}pool}}(P\setminus\{u\}) < s_{v}^{\mathrm{tf\text{-}pool}}(P)$.
	\end{enumerate}
\end{lemma}

\begin{proof}
	By Assumptions~\ref{assump:data-generating process} and \ref{assump:env-proportions}, the environment-wise sample covariances converge to their population counterparts.
	Since
	$n_e/N\to\pi_e>0$ for every $e\in\mathcal E$, it follows that, for every
	fixed $P\subseteq Z$,
	\[
		\widehat{\sigma}^{2}_{v\mid P} \xrightarrow{p} \sigma^{2}_{v\mid P,\mathrm{pool}}.
	\]
	Since $Z$ is fixed, there are only finitely many candidate parent sets, and this convergence holds simultaneously over all $P\subseteq Z$.

	For the first statement, suppose that $\sigma^{2}_{v\mid Q,\mathrm{pool}} < \sigma^{2}_{v\mid P,\mathrm{pool}}$. Then
	\[
		\frac{1}{N_v} \left( s_{v}^{\mathrm{tf\text{-}pool}}(Q) - s_{v}^{\mathrm{tf\text{-}pool}}(P) \right) \xrightarrow{p} \frac{1}{2} \log \frac{ \sigma^{2}_{v\mid Q,\mathrm{pool}} }{ \sigma^{2}_{v\mid P,\mathrm{pool}} } <0.
	\]
	The fit contribution is of order $N_v$, whereas the difference in complexity penalties is of order $\log N$, and $N_v/N\to\alpha_v>0$.
	Therefore, any move that strictly improves the population fit decreases the empirical local score with probability tending to one.

	For the second statement, suppose that deleting $u$ leaves the population residual risk unchanged.
	The two nested local models then attain the same population optimum for the centered valid-environment criterion.
	Under Assumptions~\ref{assump:data-generating process}--\ref{assump:env-proportions}, a standard fixed-dimensional quadratic expansion of the maximized criterion gives an $O_p(1)$ fit improvement from retaining the redundant parent $u$.
	Retaining $u$, however, incurs the additional penalty $\frac{1}{2}\log N$, which diverges.
	Therefore,
	\[
		s_{v}^{\mathrm{tf\text{-}pool}}(P\setminus\{u\}) < s_{v}^{\mathrm{tf\text{-}pool}}(P).
	\]
\end{proof}

\subsection{Proof of Proposition~\ref{prop:no-local-traps}}
\label{app:proof-prop-no-local-traps}

\begin{proof}
	Fix \(v\in V\) and \(Z\subseteq V\setminus\{v\}\). Under Assumptions~\ref{assump:data-generating process}, \ref{assump:conservative intervention design} and \ref{assump:env-proportions}, the valid-environment pooled covariance matrix $\Sigma_v^{\mathrm{pool}} = \frac{1}{\alpha_v}\sum_{e\in\mathcal{E}_v}\pi_e\Sigma^{(e)}$ is positive definite, since each \(\Sigma^{(e)}> 0\) and \(\pi_e/\alpha_v>0\) for every \(e\in\mathcal{E}_v\). Hence every principal submatrix of \(\Sigma_v^{\mathrm{pool}}\), in particular \(\Sigma_{ZZ}^{\mathrm{pool}}\), is positive definite and therefore invertible.

	Define the pooled population regression coefficient vector of \(X_v\) on \(X_Z\) by
	\[
		w_Z := (\Sigma_{ZZ}^{\mathrm{pool}})^{-1}\Sigma_{Zv}^{\mathrm{pool}},
		\qquad
		S := \operatorname{supp}(w_Z).
	\]
	Since \(\Sigma_{ZZ}^{\mathrm{pool}}> 0\), the least-squares coefficient vector \(w_Z\) is unique. For each \(P\subseteq Z\), let \(R(P):=\sigma^2_{v\mid P,\mathrm{pool}}\) denote the pooled population residual variance from regressing \(X_v\) on \(X_P\).

	We first show that \(R(P)=R(Z)\) if and only if \(S\subseteq P\). If \(S\subseteq P\), then \(w_Z\) vanishes on \(Z\setminus P\), so the optimal predictor based on \(X_Z\) depends only on \(X_P\). Hence \(R(P)\le R(Z)\). Since \(P\subseteq Z\), enlarging the regressor set cannot increase the least-squares residual variance, so \(R(Z)\le R(P)\). Therefore \(R(P)=R(Z)\).

	Conversely, suppose that \(R(P)=R(Z)\). Extend the least-squares coefficient vector for the regression on \(X_P\) to a vector indexed by \(Z\) by assigning zero coefficients to \(Z\setminus P\). This extended vector achieves residual variance \(R(Z)\), and is therefore also a minimizer of the least-squares problem over \(X_Z\). By uniqueness of \(w_Z\), it must coincide with \(w_Z\). Thus \((w_Z)_u=0\) for every \(u\in Z\setminus P\), which implies \(S\subseteq P\).

	Therefore, the minimizers of \(R(P)\) over \(P\subseteq Z\) are exactly the supersets of \(S\). In particular, \(S\) is the unique minimizer of smallest cardinality. We have $M_v^{\mathrm{pool}}(Z) = S = \operatorname{supp}(w_Z)$.
	It remains to prove the absence of local traps. Let \(P\subseteq Z\) satisfy \(M_v^{\mathrm{pool}}(Z)\nsubseteq P\), and set \(A:=M_v^{\mathrm{pool}}(Z)\setminus P\), so that \(A\neq\emptyset\). Suppose, toward a contradiction, that \(R(P\cup\{u\})=R(P)\) for every \(u\in A\).

	For \(u\in A\), let \(r_v^P\) and \(r_u^P\) denote the population residuals obtained by projecting \(X_v\) and \(X_u\), respectively, onto \(X_P\) under the pooled distribution. By the Frisch--Waugh--Lovell theorem~\citep{lovell2008simple}, equivalently the Schur complement identity,
	\[
		R(P)-R(P\cup\{u\}) = \frac{\operatorname{Cov}_{\mathrm{pool}}(r_v^P,r_u^P)^2} {\operatorname{Var}_{\mathrm{pool}}(r_u^P)}.
	\]
	Since \(\Sigma_v^{\mathrm{pool}}> 0\), we have \(\operatorname{Var}_{\mathrm{pool}}(r_u^P)>0\). Hence the assumption \(R(P\cup\{u\})=R(P)\) implies \(\operatorname{Cov}_{\mathrm{pool}}(r_v^P,r_u^P)=0\) for every \(u\in A\).

	Let \(r_A^P\) denote the vector of residuals obtained by projecting \(X_A\) onto \(X_P\). Then \(\operatorname{Cov}_{\mathrm{pool}}(r_A^P,r_v^P)=\mathbf{0}\). Applying the multivariate Frisch--Waugh--Lovell identity gives
	\[
		R(P)-R(P\cup A)
		=
		\operatorname{Cov}_{\mathrm{pool}}(r_v^P,r_A^P)
		\bigl[\operatorname{Var}_{\mathrm{pool}}(r_A^P)\bigr]^{-1}
		\operatorname{Cov}_{\mathrm{pool}}(r_A^P,r_v^P)
		=0,
	\]
	where \(\operatorname{Var}_{\mathrm{pool}}(r_A^P)> 0\) by positive definiteness of the corresponding Schur complement. Thus \(R(P\cup A)=R(P)\).

	On the other hand, $P\cup A = P\cup M_v^{\mathrm{pool}}(Z) \supseteq M_v^{\mathrm{pool}}(Z) = S$.
	By the characterization above, this implies \(R(P\cup A)=R(Z)\). Therefore \(R(P)=R(Z)\), which in turn implies \(S\subseteq P\). Since \(S=M_v^{\mathrm{pool}}(Z)\), this contradicts the assumption \(M_v^{\mathrm{pool}}(Z)\nsubseteq P\).

	Hence there exists at least one \(u\in M_v^{\mathrm{pool}}(Z)\setminus P\) such that $\sigma^2_{v\mid P\cup\{u\},\mathrm{pool}} < \sigma^2_{v\mid P,\mathrm{pool}}$.
\end{proof}

\subsection{Proof of Theorem~\ref{thm:iflop-gs-consistency}}

\begin{proof}
	Let $P_{\mathrm{grow}}$ denote the parent set obtained at the end of the grow phase.
	We first show that $M\subseteq P_{\mathrm{grow}}$.
	Suppose otherwise that $M\nsubseteq P_{\mathrm{grow}}$.
	By Proposition~\ref{prop:no-local-traps}, there exists $u\in M\setminus P_{\mathrm{grow}}$ such that
	\[
		\sigma^{2}_{v\mid P_{\mathrm{grow}}\cup\{u\},\mathrm{pool}} < \sigma^{2}_{v\mid P_{\mathrm{grow}},\mathrm{pool}}.
	\]
	While Lemma~\ref{lem:tf-local-score-comparison} gives
	\[
		s_{v}^{\mathrm{tf\text{-}pool}} \bigl(P_{\mathrm{grow}}\cup\{u\}\bigr) < s_{v}^{\mathrm{tf\text{-}pool}} \bigl(P_{\mathrm{grow}}\bigr),
	\]
	which contradicts the grow phase.
	Hence, $M\subseteq P_{\mathrm{grow}}$.

	We next show inductively that the shrink phase preserves $M$ while removing all variables outside $M$.
	Initially, $M\subseteq P_{\mathrm{grow}}$.
	Suppose that the current parent set $P$ during the shrink phase satisfies $M\subseteq P$.
	Since $M$ is a global minimizer of the pooled residual-risk criterion and adding nodes cannot increase the optimal population residual risk, then we have
	\[
		\sigma^{2}_{v\mid P,\mathrm{pool}} = \sigma^{2}_{v\mid M,\mathrm{pool}}.
	\]

	If $M\subsetneq P$, take any $u\in P\setminus M$.
	Since $M\subseteq P\setminus\{u\}\subseteq P$, the same monotonicity and optimality argument gives
	\[
		\sigma^{2}_{v\mid P\setminus\{u\},\mathrm{pool}} = \sigma^{2}_{v\mid P,\mathrm{pool}} = \sigma^{2}_{v\mid M,\mathrm{pool}}.
	\]
	By Lemma~\ref{lem:tf-local-score-comparison}, deleting $u$ decreases the empirical local score:
	\[
		s_{v}^{\mathrm{tf\text{-}pool}}(P\setminus\{u\}) < s_{v}^{\mathrm{tf\text{-}pool}}(P).
	\]
	Thus, every strict superset of $M$ admits a score-decreasing deletion of an extraneous parent.
	It remains to show that no score-decreasing deletion removes a member of $M$.
	Let $u\in M$.
	Since $M\subseteq P$, we have
	\[
		M\setminus(P\setminus\{u\})=\{u\}.
	\]
	Suppose, for contradiction, that $P\setminus\{u\}$ attains the same minimum population residual risk as $M$.
	Since it does not contain $M$, the no-local-trap assumption implies that adding $u$ strictly reduces its population residual risk:
	\[
		\sigma^{2}_{v\mid P,\mathrm{pool}} = \sigma^{2}_{v\mid (P\setminus\{u\})\cup\{u\},\mathrm{pool}} < \sigma^{2}_{v\mid P\setminus\{u\},\mathrm{pool}}.
	\]
	which contradicts the assumed equality with the minimum risk attained by $M$.
	Therefore,
	\[
		\sigma^{2}_{v\mid P\setminus\{u\},\mathrm{pool}} > \sigma^{2}_{v\mid P,\mathrm{pool}}.
	\]
	Applying the first part of Lemma~\ref{lem:tf-local-score-comparison}, with $Q=P$ and with $P$ there replaced by $P\setminus\{u\}$, gives
	\[
		s_{v}^{\mathrm{tf\text{-}pool}}(P) < s_{v}^{\mathrm{tf\text{-}pool}}(P\setminus\{u\}).
	\]
	Hence deleting a node in $M$ does not improve the score.

	It follows inductively that every accepted shrink move removes a node in $P_{\mathrm{grow}}\setminus M$, and the shrink phase terminates at $M$.

	Therefore, one complete grow--shrink cycle returns $M_v^{\mathrm{pool}}(Z)$ in the limit $N\to\infty$.
	Since the argument does not depend on the initial parent set $P_0\subseteq Z$, target-filtered grow--shrink returns $M_v^{\mathrm{pool}}(Z)$ with probability tending to one from any warm start.
\end{proof}

\subsection{Proof of Corollary~\ref{cor:iflop-fixed-order}}

\begin{proof}
	Fix an order $\tau$.
	For every node $v\in V$, let $Z_v(\tau)=\operatorname{Pred}_{\tau}(v)$ be the set of variables preceding $v$ in $\tau$.
	Then applying Theorem~\ref{thm:iflop-gs-consistency} to each node $v$ gives
	\[
		\Pr\!\left( \widehat P_{v,N}(\tau) = M_v^{\mathrm{pool}}(Z_v(\tau)) \right) \longrightarrow 1.
	\]
	Since $V$ is finite, these events hold simultaneously for all nodes with probability tending to one.
	Therefore,
	\[
		\Pr\!\left( \widehat P_{v,N}(\tau) = M_v^{\mathrm{pool}}(Z_v(\tau)) \ \text{for all } v\in V \right) \longrightarrow 1.
	\]

	By construction, the DAG $\widehat G_N(\tau)$ returned under the fixed order $\tau$ has parent sets
	\[
		\operatorname{Pa}_{\widehat G_N(\tau)}(v) = \widehat P_{v,N}(\tau), \qquad v\in V.
	\]
	It follows that
	\[
		\operatorname{Pa}_{\widehat G_N(\tau)}(v) = M_v^{\mathrm{pool}}(Z_v(\tau)), \qquad v\in V,
	\]
	simultaneously with probability tending to one.
\end{proof}

\subsection{Proof of Lemma~\ref{lem:likelihood-compatibility}}

\begin{proof}
	The local centered likelihood contribution of the original mechanism of $v$ over its valid environments is
	\[
		\ell_N(P;\beta_v,\sigma_v^2) = \sum_{e\in\mathcal E_v} \ell_{v,e}(P;\beta_v,\sigma_v^2) = -\frac{N_v}{2}\log(2\pi\sigma_v^2) -\frac{1}{2\sigma_v^2} \sum_{e\in\mathcal E_v} \left\| \widetilde X_v^{(e)} - \widetilde X_P^{(e)}\beta_v \right\|_2^2.
	\]
	For fixed $P$, maximizing over $\beta_v$ yields the pooled least-squares estimator computed from the valid environments of $v$.
	Maximizing further over $\sigma_v^2$ gives
	\[
		\widehat{\sigma}^{2}_{v\mid P} = \frac{1}{N_v} \sum_{e\in\mathcal E_v} \left\| \widetilde X_v^{(e)} - \widetilde X_P^{(e)} \widehat{\beta}_{v\mid P} \right\|_2^2.
	\]
	Therefore,
	\[
		-\sup_{\beta_v,\sigma_v^2} \ell_N(P;\beta_v,\sigma_v^2) = \frac{N_v}{2} \left( \log(2\pi) + 1 + \log\widehat{\sigma}^{2}_{v\mid P} \right).
	\]

	Since $\ell_N(G,\theta_G)$ is node-wise decomposable,
	\[
		-\sup_{\theta_G} \ell_N(G,\theta_G) = \sum_{v\in V} \frac{N_v}{2} \left( 1+ \log \widehat{\sigma}^{2}_{v\mid\operatorname{Pa}_G(v)} \right) + \frac{1}{2} \sum_{v\in V} N_v\log(2\pi).
	\]
	Adding the complexity penalty in the target-filtered pooled score yields
	\[
		S_{\mathrm{tf\text{-}pool}}(G;\mathcal D,\mathcal I) = -\sup_{\theta_G} \ell_N(G,\theta_G) + \frac{1}{2}\log N \sum_{v\in V} \left( |\operatorname{Pa}_G(v)|+1 \right) + C_N,
	\]
	where $C_N = - \frac{1}{2} \sum_{v\in V} N_v\log(2\pi)$ does not depend on $G$.
	Thus, $S_{\mathrm{tf\text{-}pool}}$ is a decomposable criterion for the valid-environment structural-mechanism likelihood after profiling out environment-specific intercepts.
	By Assumption~\ref{assump:data-generating process}, the likelihood contributions associated with directly targeted replacement mechanisms are independent of the candidate parent sets of the original mechanisms and therefore do not affect graph comparison.
\end{proof}

\subsection{Proof of Theorem~\ref{thm:global-imec-identifiability}}

\begin{proof}
	By Lemma~\ref{lem:likelihood-compatibility}, the target-filtered pooled score agrees, up to terms independent of the candidate DAG, with the negative maximized valid-environment structural-mechanism likelihood after profiling out environment-specific intercepts.
	Under known hard interventions, the omitted likelihood contributions of directly targeted replacement mechanisms are graph-independent nuisance terms.
	Hence, $L_\infty(G)$ is the population fit criterion induced by the graph-dependent part of the known-target interventional likelihood.

	The true DAG $G^\star$ generates the observational and interventional distribution family under the known target family $\mathcal I$.
	Therefore, it attains the minimum value of $L_\infty$, and $G^\star\in\mathcal A$.

	Now let $H\in\mathcal A$.
	Since $H$ attains the same minimum value of $L_\infty$ as $G^\star$, the maximized population graph-dependent likelihood under $H$ has no loss relative to the true known-target interventional distribution family.
	Under the regular finite-dimensional Gaussian model, a misspecified candidate model would incur a strictly positive population Kullback--Leibler loss.
	Hence, the true observational distribution is Markov with respect to $H$, and, for every intervention target $I_e\in\mathcal I$, the corresponding interventional distribution is Markov with respect to the intervention graph $H^{I_e}$.
	Thus, every $H\in\mathcal A$ is a causal independence map of the true known-target interventional distribution family.

	Fix any $G\in\mathcal A_{\min}$.
	Since $G^\star\in\mathcal A$ and $\mathcal A_{\min}$ retains the sparsest members of $\mathcal A$, we have $|E_G|\le |E_{G^\star}|$.
	By the interventional faithfulness assumption, \citet[Remark~1]{hauser2015jointly} show that the minimum causal independence maps of the true known-target interventional distribution family are exactly the DAGs in the $\mathcal I$-Markov equivalence class of $G^\star$.
	In particular, $G^\star$ is a minimum causal independence map, so no causal independence map of this distribution family has fewer edges than $G^\star$.

	Since $G\in\mathcal A_{\min}\subseteq\mathcal A$, the preceding argument shows that $G$ is itself a causal independence map.
	Therefore, $|E_G|\ge |E_{G^\star}|$.

	Combining the two inequalities yields $|E_G|=|E_{G^\star}|$.

	Hence, $G$ is a minimum causal independence map of the true known-target interventional distribution family and $G\in\mathcal G_{\mathcal I}^{\star}$.

	Since $G\in\mathcal A_{\min}$ was arbitrary, we conclude that
	\[
		\mathcal A_{\min} \subseteq \mathcal G_{\mathcal I}^{\star}.
	\]

\end{proof}

\subsection{Proof of Corollary~\ref{cor:exact-icpdag-recovery}}

\begin{proof}
	Let
	\[
		\widehat G_N \in \operatorname*{arg\,min}_{G\in\mathcal G_{\mathrm{DAG}}} S_{\mathrm{tf\text{-}pool}}(G;\mathcal D,\mathcal I)
	\]
	be an exact global minimizer of the empirical target-filtered pooled score.
	Define its empirical fit component by
	\[
		R_N(G) = \sum_{v\in V} \frac{N_v}{2} \log \widehat{\sigma}^{2}_{v\mid\operatorname{Pa}_G(v)}.
	\]
	Up to terms independent of $G$,
	\[
		S_{\mathrm{tf\text{-}pool}}(G;\mathcal D,\mathcal I) = R_N(G) + \frac{1}{2}\log N \sum_{v\in V} \left(|\operatorname{Pa}_G(v)|+1\right).
	\]

	For every fixed DAG $G$, consistency of the valid-environment pooled residual variances and the assumption on the asymptotic environment sample proportions imply
	\[
		\frac{1}{N}R_N(G) \xrightarrow{p} L_\infty(G).
	\]
	Since $V$ is fixed, $\mathcal G_{\mathrm{DAG}}$ is finite, and this convergence holds simultaneously over all candidate DAGs.
	Consider any $G\notin\mathcal A$ and any $G_0\in\mathcal A$.
	By the definition of $\mathcal A$, we have $L_\infty(G)>L_\infty(G_0)$.
	Therefore,
	\[
		R_N(G)-R_N(G_0) = N\left(L_\infty(G)-L_\infty(G_0)\right) + o_p(N).
	\]
	The difference in complexity penalties is of order $\log N$, whereas the fit difference is of order $N$.
	Hence,
	\[
		S_{\mathrm{tf\text{-}pool}}(G;\mathcal D,\mathcal I) - S_{\mathrm{tf\text{-}pool}}(G_0;\mathcal D,\mathcal I) \xrightarrow{p} +\infty.
	\]
	Since the DAG space is finite, it follows that $\Pr\!\left(\widehat G_N\in\mathcal A\right) \longrightarrow 1$.

	It remains to distinguish among the members of $\mathcal A$.
	Let $G\in\mathcal A\setminus\mathcal A_{\min}$ and $G_0\in\mathcal A_{\min}$.
	Since $G,G_0\in\mathcal A$, both attain the same value of $L_\infty$ and, by the proof of Theorem~\ref{thm:global-imec-identifiability}, are causal independence maps of the true known-target interventional distribution family.
	Under Assumptions~\ref{assump:data-generating process}, \ref{assump:local-nonsingularity}, and \ref{assump:env-proportions}, the difference between their maximized valid-environment structural-mechanism likelihood contributions is $O_p(1)$ for the fixed-dimensional Gaussian conditional-mechanism models considered here.
	Equivalently,
	\[
		R_N(G)-R_N(G_0)=O_p(1).
	\]

	Since $G\notin\mathcal A_{\min}$ and $G_0\in\mathcal A_{\min}$, we have $|E_G|>|E_{G_0}|$.
	Moreover,
	\[
		\sum_{v\in V} \left( |\operatorname{Pa}_G(v)|+1\right) = |E_G|+|V|.
	\]
	Hence, the complexity penalty difference between $G$ and $G_0$ is $\frac{1}{2}\log N \left( |E_G|-|E_{G_0}| \right)$, which diverges to $+\infty$.
	Therefore,
	\[
		\Pr\!\left( S_{\mathrm{tf\text{-}pool}}(G;\mathcal D,\mathcal I) > S_{\mathrm{tf\text{-}pool}}(G_0;\mathcal D,\mathcal I) \right) \longrightarrow 1.
	\]
	Again using finiteness of $\mathcal G_{\mathrm{DAG}}$, we obtain $\Pr\!\left( \widehat G_N\in\mathcal A_{\min} \right) \longrightarrow 1$.
	By Theorem~\ref{thm:global-imec-identifiability}, we also have $\mathcal A_{\min} \subseteq \mathcal G_{\mathcal I}^{\star}$.
	Consequently, $\Pr\!\left( \widehat G_N\in\mathcal G_{\mathcal I}^{\star} \right) \longrightarrow 1$.

	If the known target family fully identifies the true DAG, so that $\mathcal G_{\mathcal I}^{\star} = \{G^\star\}$, then
	\[
		\Pr\!\left( \widehat G_N=G^\star \right) \longrightarrow 1.
	\]
\end{proof}

\section{Detailed Algorithmic Specifications}
\label{app:algorithm-details}

Algorithms~\ref{alg:reinsert-tf} and~\ref{alg:growshrink-tf}
detail the reinsertion and warm-start parent-set updates invoked in
Algorithm~\ref{alg:iflop}.
For brevity, let
\[
	s_v := s_v^{\mathrm{tf\text{-}pool}},
	\qquad
	P := (P_x)_{x\in V},
	\qquad
	\ell := (\ell_x)_{x\in V},
\]
where $P_x$ is the current parent set of node $x$ and
$\ell_x=s_x(P_x)$ is its current local score. We write
\[
	Q(\ell):=\sum_{x\in V}\ell_x
\]
for the total score.

For a prefix change of node $x$, $+x$ means that $x$ becomes admissible
as a parent, whereas $-x$ means that $x$ ceases to be admissible as a
parent. To describe the prefix changes induced by an adjacent swap, define
\[
	\Delta(d,x)
	:=
	\begin{cases}
		+x, & d=+1, \\
		-x, & d=-1,
	\end{cases}
	\qquad d\in\{-1,+1\}.
\]

\begin{algorithm2e}[H]
	\SetAlgoLined
	\DontPrintSemicolon
	\caption{Target-filtered reinsertion of node $v$}
	\label{alg:reinsert-tf}

	\KwIn{Current order $\tau$, parent sets $P$, local scores $\ell$, node $v$}
	\KwOut{Best reinsertion state $(\widehat{\tau},\widehat{P},\widehat{\ell})$}

	\BlankLine
	\tcp{Initialize the best state with the current state}
	$(\widehat{\tau},\widehat{P},\widehat{\ell})
		\leftarrow
		(\tau,P,\ell)$\;

	\BlankLine
	\tcp{Explore moving $v$ to the left and to the right}
	\ForEach{$d\in\{-1,+1\}$}{
		$(\tau',P',\ell')\leftarrow(\tau,P,\ell)$\;
		$j\leftarrow\operatorname{position}_{\tau'}(v)$\;

		\While{$1\le j+d\le p$}{
			$u\leftarrow\tau'_{j+d}$\;

			\tcp{Swap $v$ with the adjacent node $u$}
			$(\tau'_j,\tau'_{j+d})
				\leftarrow
				(\tau'_{j+d},\tau'_j)$\;

			\tcp{Update only the two parent sets affected by the swap}
			$(P'_v,\ell'_v)
				\leftarrow
				\mathrm{GrowShrink}_{\mathrm{tf}}
				\bigl(
				v,
				\operatorname{Pred}_{\tau'}(v),
				P'_v,
				\ell'_v,
				\Delta(d,u)
				\bigr)$\;

			$(P'_u,\ell'_u)
				\leftarrow
				\mathrm{GrowShrink}_{\mathrm{tf}}
				\bigl(
				u,
				\operatorname{Pred}_{\tau'}(u),
				P'_u,
				\ell'_u,
				\Delta(-d,v)
				\bigr)$\;

			$j\leftarrow j+d$\;

			\tcp{Retain the best reinsertion position found so far}
			\If{$Q(\ell')<Q(\widehat{\ell})$}{
				$(\widehat{\tau},\widehat{P},\widehat{\ell})
					\leftarrow
					(\tau',P',\ell')$\;
			}
		}
	}

	\BlankLine
	\KwRet{$(\widehat{\tau},\widehat{P},\widehat{\ell})$}\;
\end{algorithm2e}

\begin{algorithm2e}[H]
	\SetAlgoLined
	\DontPrintSemicolon
	\caption{Target-filtered warm-start grow--shrink}
	\label{alg:growshrink-tf}

	\SetKwFunction{GSTF}{GrowShrink$_{\mathrm{tf}}$}
	\SetKwProg{Fn}{Function}{:}{}

	\Fn{\GSTF{$v,Z,P_0,\ell_0,\delta$}}{
		\uIf{$\delta=+u$}{
			$P\leftarrow P_0\cup\{u\}$; $\ell\leftarrow s_v(P)$\;
			\If{$\ell\ge\ell_0$}{
				\KwRet{$P_0,\ell_0$}\;
			}
		}
		\uElseIf{$\delta=-u$}{
			\If{$u\notin P_0$}{
				\KwRet{$P_0,\ell_0$}\;
			}
			$P\leftarrow P_0\setminus\{u\}$; $\ell\leftarrow s_v(P)$\;
		}
		\Else{
			$P\leftarrow\emptyset$; $\ell\leftarrow s_v(\emptyset)$\;
		}

		\ForEach{$d\in\{+,-\}$}{
			\Repeat{$P=P^{\mathrm{old}}$}{
				$P^{\mathrm{old}}\leftarrow P$\;
				$\mathcal U\leftarrow
					\begin{cases}
						Z\setminus P^{\mathrm{old}}, & d=+, \\
						P^{\mathrm{old}},            & d=-;
					\end{cases}$\;

				\ForEach{$u\in\mathcal U$}{
					\If{$d=+$}{
						$P'\leftarrow P\cup\{u\}$\;
					}
					\Else{
						\If{$u\notin P$}{
							\textbf{continue}\;
						}
						$P'\leftarrow P\setminus\{u\}$\;
					}

					$\ell'\leftarrow s_v(P')$\;
					\If{$\ell'<\ell$}{
						$(P,\ell)\leftarrow(P',\ell')$\;
					}
				}
			}
		}
		\KwRet{$P,\ell$}\;
	}
\end{algorithm2e}

The $\delta=+u$ short-circuit (returning $P_0,\ell_0$ immediately when adding $u$ does not improve the score) is justified whenever $P_0$ is itself already a grow--shrink fixed point over $Z\setminus\{u\}$, which holds at every call site in Algorithm~\ref{alg:reinsert-tf}: by Proposition~\ref{prop:no-local-traps}, if no single admissible addition to such a $P_0$ improves the score, no further local move can either, so the full grow--shrink procedure would return $P_0$ unchanged regardless.

\section{Computational Scaling with Graph Size and Sample Size}
\label{scalability}

\begin{figure}[bt]
	\centering
	\includegraphics[width=\linewidth]{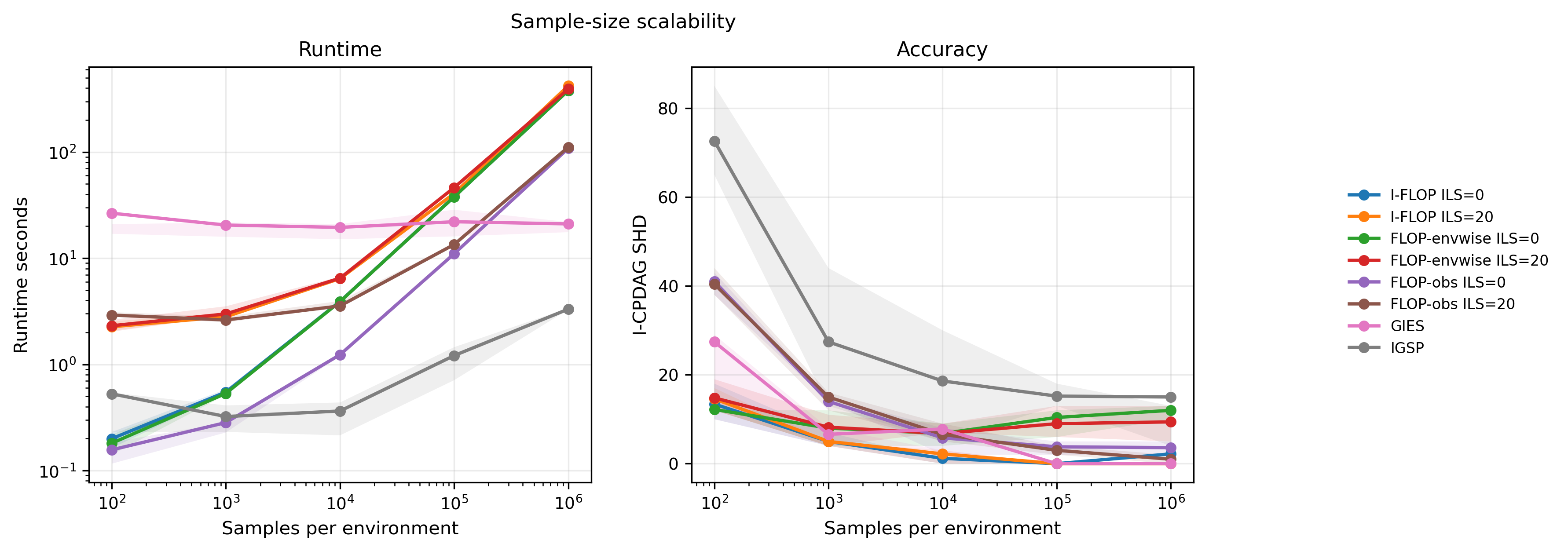}
	\includegraphics[width=\linewidth]{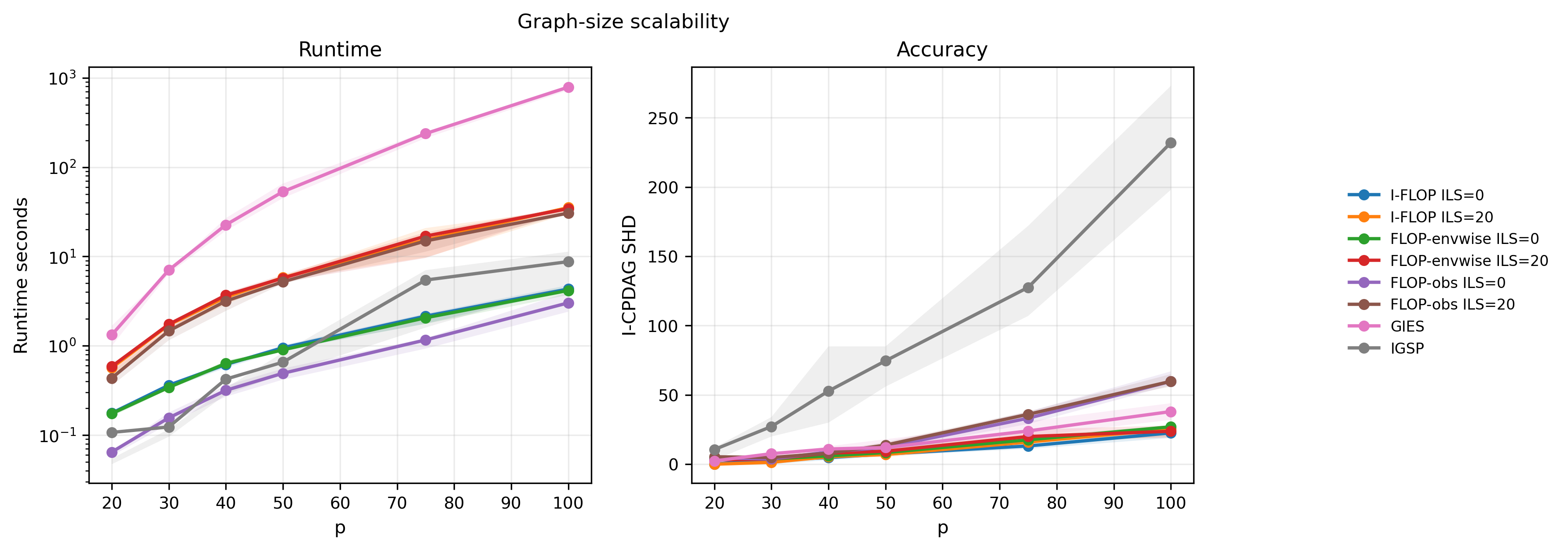}
	\caption{Top: Scaling with the number of samples per environment, with $p=40$, expected degree $3$, and sample size varied from $10^2$ to $10^6$ in each environment.
		\emph{Bottom:} Scaling with graph size, with expected degree $3$, $1000$ samples per environment, and $p$ varied from $20$ to $100$.
		Each setting contains one observational environment and five known-target singleton perfect hard-intervention environments.
		Runtime is shown on a logarithmic scale, and accuracy is measured by I-CPDAG SHD.}
	\label{fig:scalability}
\end{figure}

Figure~\ref{fig:scalability} reports additional scaling results on sparse ER DAGs.
We examine two complementary regimes: increasing the number of samples per environment while holding the graph regime fixed, and increasing the number of variables while fixing the sample size, expected degree, and intervention design.

\paragraph{Scaling with sample size.}
As the number of samples per environment increases, I-FLOP substantially reduces its I-CPDAG SHD, approaching near-perfect recovery.
GIES exhibits a comparable reduction in structural error while maintaining nearly constant runtime in this experiment.
By contrast, the runtime of the FLOP-based methods increases markedly with sample size, particularly beyond $10^4$ samples per environment.
FLOP-envwise does not match the accuracy of I-FLOP at large sample sizes, indicating that environment-wise estimation alone is insufficient without target-based filtering.
I-GSP improves with additional samples but retains noticeably higher structural error than I-FLOP and GIES.

\paragraph{Scaling with graph size.}
When the number of variables increases, I-FLOP maintains the lowest structural error among the evaluated methods while remaining computationally inexpensive relative to GIES.
Its runtime grows moderately from $p=20$ to $p=100$, whereas GIES becomes substantially more expensive, reaching several hundred seconds at the largest graph size.
The accuracy difference is also pronounced for larger graphs: I-FLOP retains comparatively low I-CPDAG SHD, while the errors of GIES, FLOP-obs, and especially I-GSP increase considerably.
FLOP-envwise remains computationally close to I-FLOP but incurs consistently higher recovery error, again supporting the contribution of target-based environment filtering.

\paragraph{Explanation of the runtime of GIES.}

\begin{figure}[ht]
	\centering
	\includegraphics[width=0.8\linewidth]{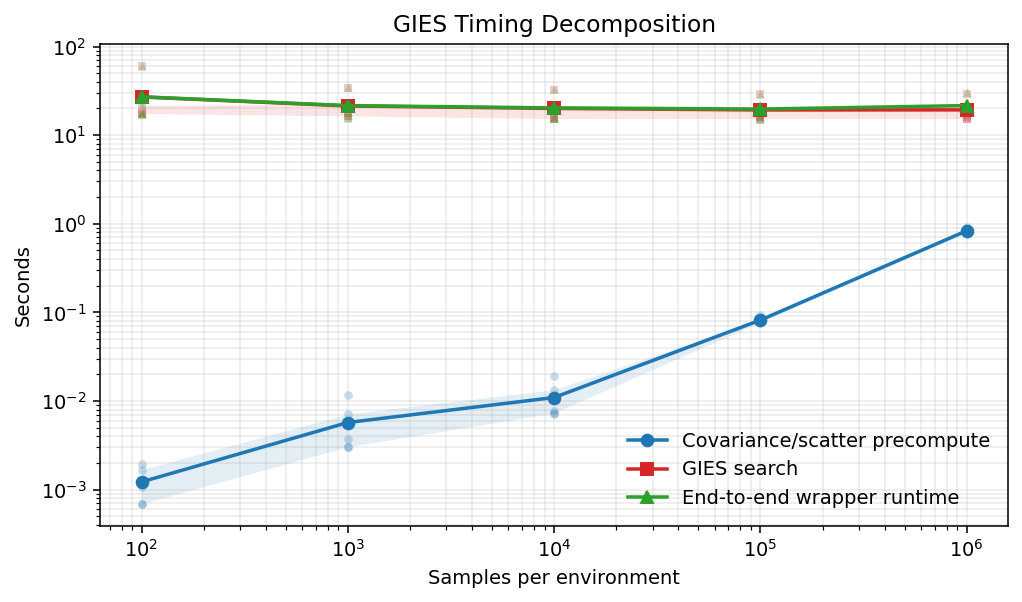}
	\caption{Runtime of GIES is decomposed into covariance/scatter precomputation, graph search, and end-to-end wrapper execution.
		All the settings are the same as Figure~\ref{fig:scalability}.}
	\label{fig:gies}
\end{figure}

The nearly flat runtime curve of GIES in the sample-size experiment is explained by the timing decomposition in Figure~\ref{fig:gies}.
GIES performs its graph search using covariance computation.
As the number of samples per environment increases from $10^2$ to $10^6$, the cost of computing these statistics rises from the millisecond scale to below one second.
However, the search stage consistently requires approximately $20$ seconds and therefore dominates the end-to-end runtime throughout the evaluated range.
End-to-end wrapper execution, covering input conversion into the GIES implementation's required format and output conversion back into a standardized DAG representation for evaluation, contributes only a small, roughly constant overhead relative to the search stage.
Consequently, increasing the sample size has little visible effect on total GIES runtime in Figure~\ref{fig:scalability}.
The small non-monotonic variation at lower sample sizes is attributable to changes in the estimated scores, which will alter the greedy search trajectory and its number of steps, rather than to a systematic reduction in computational cost.

\section{Diagnostic ILS Trace on the Real Data Experiment}
\label{ILS}

To investigate the observed mismatch between score improvement and graph recovery as the ILS budget increases, we conducted an additional diagnostic analysis on the Chamber dataset from the real-data experiment (Section~\ref{causalchamber}).
This analysis follows exactly the same experimental setting as the original real-data evaluation and examines how the optimized total BIC score and the corresponding SHD change as the ILS budget increases.

\begin{figure}[t]
	\centering
	\begin{minipage}[t]{0.49\linewidth}
		\centering
		\includegraphics[width=\linewidth]{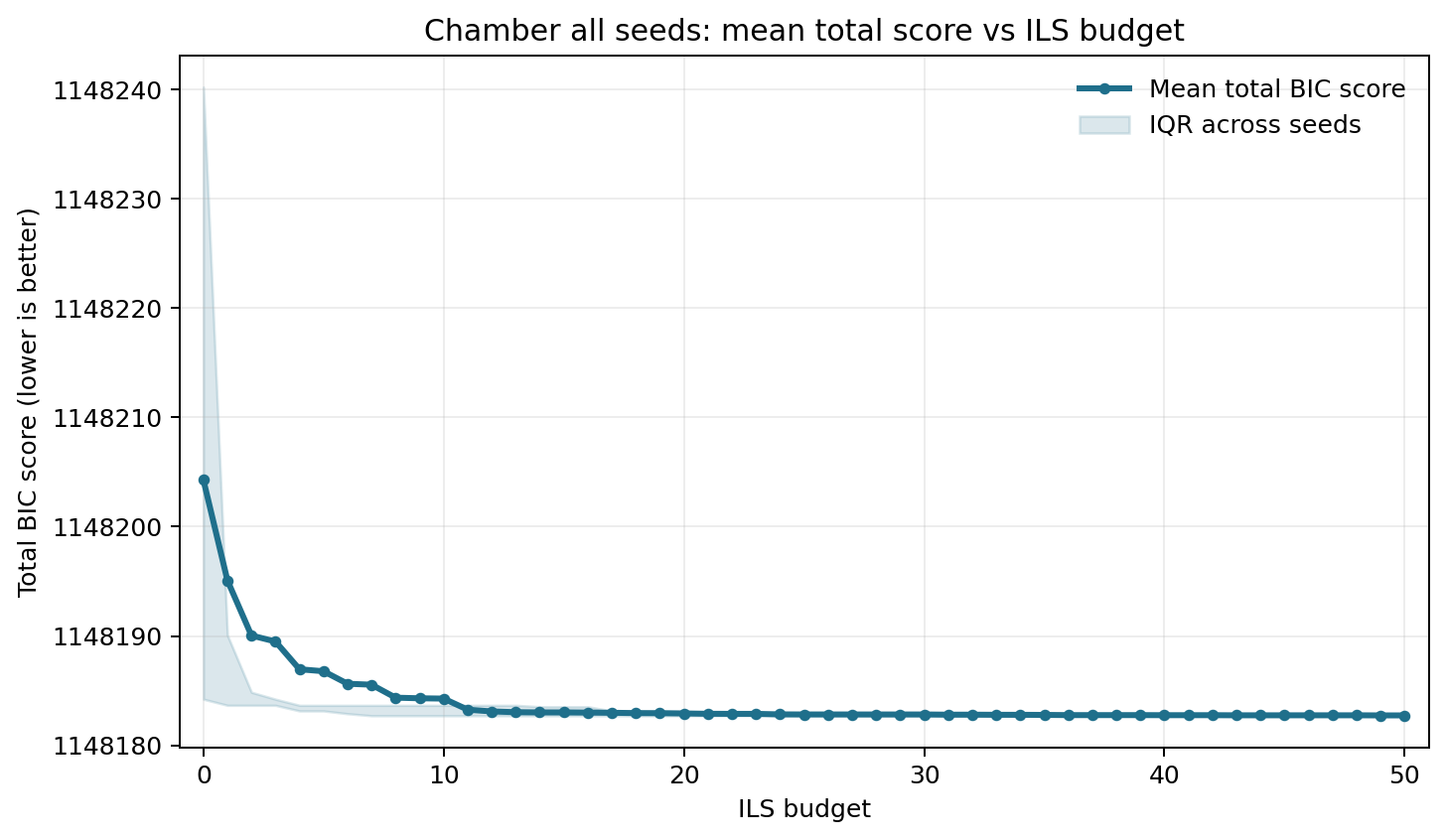}
	\end{minipage}
	\hfill
	\begin{minipage}[t]{0.49\linewidth}
		\centering
		\includegraphics[width=\linewidth]{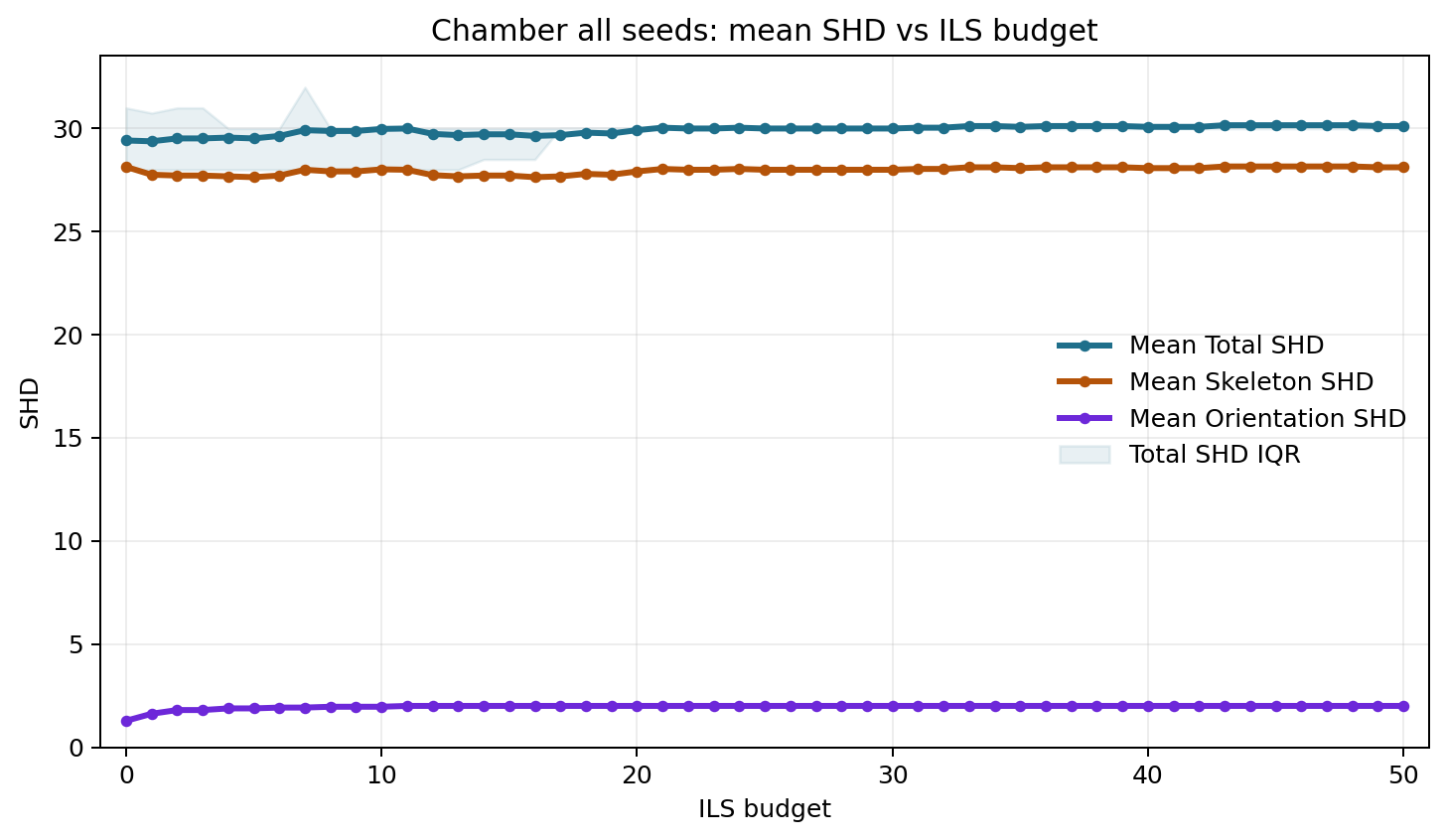}
	\end{minipage}
	\caption{\textbf{Left:} Mean optimized total BIC score across 50 seeds as the ILS budget increases from 0 to 50.
		\textbf{Right:} Mean total SHD and its skeleton and orientation components across the same seeds.
		The shaded regions denote the interquartile range across seeds.}
	\label{fig:chamber_ils_diagnostic}
\end{figure}

As shown in Figure~\ref{fig:chamber_ils_diagnostic} (left), the mean total score decreases rapidly during the first several ILS steps and then gradually stabilizes.
From ILS 0 to ILS 50, the mean total BIC score decreases by approximately 21.55 units, and 49 of the 50 seeds achieve a lower score at ILS 50 than at ILS 0.
However, this improvement in the optimized total score does not translate into better graph recovery.
Figure~\ref{fig:chamber_ils_diagnostic} (right) shows that the mean total SHD increases from 29.42 at ILS 0 to 30.12 at ILS 50.
This increase is almost entirely attributable to the orientation component: the mean skeleton SHD remains essentially unchanged, decreasing slightly from 28.14 to 28.12, whereas the mean orientation SHD increases from 1.28 to 2.00.

\begin{figure}[t]
	\centering
	\includegraphics[width=0.7\linewidth]{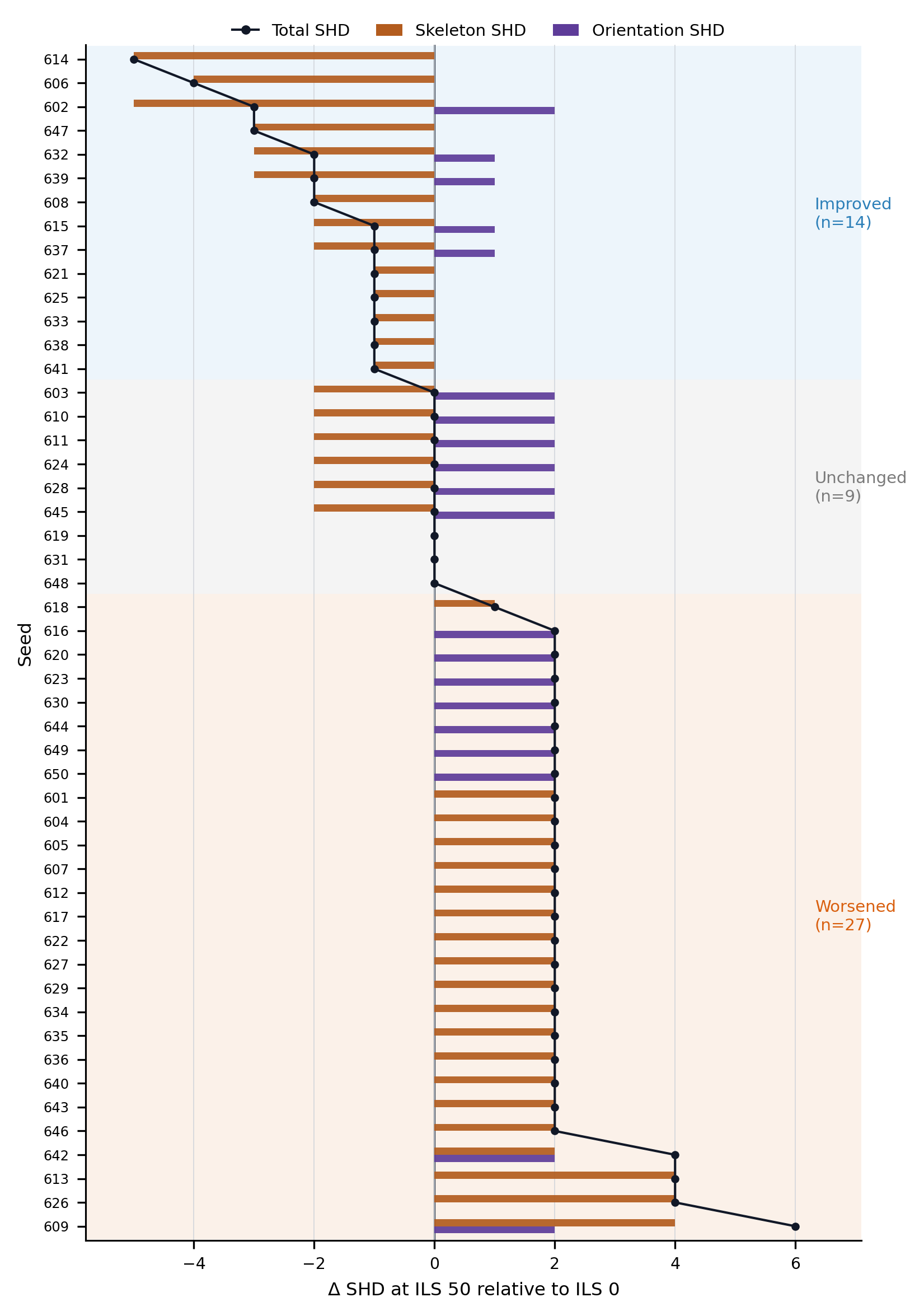}
	\caption{Seed-level changes in SHD on the real data dataset when increasing the ILS budget from 0 to 50.
		The horizontal axis reports the change in SHD, computed as the value at ILS 50 minus that at ILS 0.
		Seeds are grouped according to the change in total SHD: improved ($n=14$), unchanged ($n=9$), and worsened ($n=27$).
		The black line denotes the change in total SHD, while the orange and purple bars denote the changes in skeleton SHD and orientation SHD, respectively.}
	\label{fig:chamber_ils_seed_level_shd}
\end{figure}

However, the aggregate results mask substantial heterogeneity across seeds.
As shown in Figure~\ref{fig:chamber_ils_seed_level_shd}, comparing ILS 0 with ILS 50, total SHD worsens for 27 seeds, improves for 14 seeds, and remains unchanged for 9 seeds.
Skeleton SHD increases for 20 seeds and decreases for 20 seeds, whereas orientation SHD increases for 20 seeds and decreases for none.
Although skeleton changes constitute the larger component-wise changes for 33 seeds, they occur in both directions and therefore largely cancel in the aggregate mean.
By contrast, orientation changes are one-sided and consequently account for the net increase in mean total SHD.
Thus, in the real data experiment, increasing the ILS budget generally yields lower-scoring solutions, but these solutions are not necessarily closer to the reference graph under SHD.

This score–recovery mismatch is plausibly related to a mismatch between the assumptions of the I-FLOP algorithm and the data-generating process of the Causal Chamber dataset.
In the synthetic experiments, the data are generated from a linear-Gaussian SEM, and interventions are implemented as Gaussian shifts on target variables while preserving the same linear mechanisms for the remaining variables.
However, Causal Chamber is obtained from a real optical system rather than from this synthetic SEM.
Its intervention environments are created by changing the operating ranges of physical actuators to bounded intervention-specific ranges.
These interventions therefore alter the support and operating regime of the intervened variables, rather than merely applying Gaussian mean shifts within the same assumed model.

These differences affect the interpretation of the optimized score.
The I-FLOP score evaluates candidate graphs through linear-Gaussian regressions pooled across the relevant observational and interventional environments.
Consequently, a lower total BIC score on the real data experiment represents improved fit under the assumed linear-Gaussian scoring approximation, rather than direct evidence of improved recovery of the physical reference graph.
The ILS results suggest that increasing the search budget improves the optimized score, while the corresponding solutions retain essentially unchanged skeleton accuracy and accumulate additional orientation errors.

\end{document}